\documentclass[final,onefignum,onetabnum]{siamart220329}

\usepackage{lipsum}
\usepackage{amsfonts}
\usepackage{graphicx}
\usepackage{epstopdf}
\usepackage[algo2e]{algorithm2e}
\usepackage{bm}
\usepackage{cases}
\usepackage{multirow}

\ifpdf
  \DeclareGraphicsExtensions{.eps,.pdf,.png,.jpg}
\else
  \DeclareGraphicsExtensions{.eps}
\fi

\newsiamremark{remark}{Remark}
\newsiamremark{hypothesis}{Hypothesis}
\crefname{hypothesis}{Hypothesis}{Hypotheses}
\newsiamthm{claim}{Claim}

\headers{Quaternion nuclear norms over Frobenius norms}{Y. Guo, G. Chen, T. Zeng, Q. Jin, and M. Ng}

\title{Quaternion Nuclear Norms Over Frobenius Norms Minimization for Robust Matrix Completion
\thanks{Submitted to the editors 2024-11-.
\funding{This work was supported by National Natural Science Foundation of China (No. 12061052), Young Talents of Science and Technology in Universities of Inner Mongolia Autonomous Region (No. NJYT22090), Natural Science Fund of Inner Mongolia Autonomous Region (No. 2020MS01002), Innovative Research Team in Universities of Inner Mongolia Autonomous Region (No. NMGIRT2207),  Special Funds for Graduate Innovation and Entrepreneurship of Inner Mongolia University (No.~11200-121024), Prof. Guoqing Chen's “111 project” of higher education talent training in Inner Mongolia Autonomous Region, Inner Mongolia University Independent Research Project (No. 2022-ZZ004) and the network information center of Inner Mongolia University. M. Ng’s research is funded by HKRGC GRF 17201020 and 17300021, HKRGC CRF C7004-21GF, and Joint NSFC and RGC N-HKU769/21.}
}
}

\author{Yu Guo \thanks{School of Mathematical Science, Inner Mongolia University, Hohhot, China. Corresponding author: Qiyu Jin. 
  (\email{yuguomath@aliyun.com}, \email{cgq@imu.edu.cn}, \email{qyjin2015@aliyun.com}).}
\and Guoqing Chen\footnotemark[2]
\and Tieyong Zeng  \thanks{Department of Mathematics, The Chinese University of Hong Kong, Satin, Hong Kong (\email{zeng@math.cuhk.edu.hk}).}
\and Qiyu Jin\footnotemark[2]
\and Michael Kwok-Po Ng\thanks{Department of Mathematics, Hong Kong Baptist University, Kowloon Tong, Hong Kong (\email{michael-ng@hkbu.edu.hk}).}
}

\usepackage{amsopn}
\DeclareMathOperator{\diag}{diag}

\ifpdf
\hypersetup{
  pdftitle={Quaternion nuclear norms over Frobenius norms},
  pdfauthor={Y. Guo, G. Chen, T. Zeng, Q. Jin, and M. Ng}
}
\fi

\begin{document}

\maketitle

% REQUIRED
\begin{abstract}
Recovering hidden structures from incomplete or noisy data remains a pervasive challenge across many fields, particularly where multi-dimensional data representation is essential. Quaternion matrices, with their ability to naturally model multi-dimensional data, offer a promising framework for this problem. This paper introduces the quaternion nuclear norm over the Frobenius norm (QNOF) as a novel nonconvex approximation for the rank of quaternion matrices. QNOF is parameter-free and scale-invariant. Utilizing quaternion singular value decomposition, we prove that solving the QNOF can be simplified to solving the singular value $L_1/L_2$ problem. Additionally, we extend the QNOF to robust quaternion matrix completion, employing the alternating direction multiplier method to derive solutions that guarantee weak convergence under mild conditions. Extensive numerical experiments validate the proposed model's superiority, consistently outperforming state-of-the-art quaternion methods.
\end{abstract}

% REQUIRED
\begin{keywords}
quaternion, color image inpainting, sparsity, ADMM, scale-invariant
\end{keywords}

% REQUIRED
\begin{MSCcodes}
65F35, 90C30, 94A08, 68U10
\end{MSCcodes}

\section{Introduction}

%In recent years, the problem of recovering low-rank matrices from degenerate observations has attracted significant attention. This issue has a wide range of applications in fields such as computer vision \cite{gu2017weighted}, image processing \cite{guo2022gaussian,li2022selecting}, and machine learning \cite{udell2019big}. As multidimensional data becomes more prevalent, additional tools for characterizing such data, including tensors and quaternions, are increasingly adopted.  Consequently, low-rank matrix recovery has been extended to low-rank tensor recovery \cite{zhang2019corrected,zhao2022robust} and quaternion low-rank matrix recovery \cite{jia2019robust,guo2025quaternion}.

Recovering low-rank matrices from incomplete or corrupted observations is a fundamental problem with wide-ranging applications in computer vision \cite{gu2017weighted}, image processing \cite{guo2022gaussian,li2022selecting}, and machine learning \cite{udell2019big} and so on. The prevalence of multidimensional data in modern applications has further driven the development of sophisticated tools such as tensors and quaternions for effectively representing and processing such data. Consequently, the scope of low-rank recovery has expanded beyond traditional matrices to include low-rank tensor recovery \cite{zhang2019corrected, zhao2022robust} and quaternion low-rank matrix recovery \cite{jia2019robust, guo2025quaternion}.

The problem of minimizing the rank of a matrix is expressed as:
\begin{equation}
 \min_{\bm{X}} \mathrm{rank}(\bm{X}) \quad \mathrm{s.t.} \quad \mathcal{A}(\bm{X}) = \bm{Y},
 \label{Model lowrank}
\end{equation}
%where $\bm{X} \in \mathbb{R}^{m \times n} $ is an unknown low-rank matrix, and $\bm{Y}$ is its corresponding observation. $\mathcal{A}$ is the linear map. For low-rank matrix recovery, rank minimization is usually NP-hard \cite{natarajan1995sparse}. 
where $\bm{X} \in \mathbb{R}^{m \times n}$ is the unknown low-rank matrix, $\bm{Y}$ denotes the observed data, and $\mathcal{A}$ is a linear mapping. However, solving this rank minimization problem is generally NP-hard \cite{natarajan1995sparse}. To overcome this difficulty, the nuclear norm, the tightest convex relaxation of the rank function \cite{recht2010guaranteed}, is often used as an approximation. The nuclear norm minimization (NNM) is described as
\begin{equation}
 \min_{\bm{X}} || \bm{X} ||_{*} \quad \mathrm{s.t.} \quad \mathcal{A}(\bm{X}) = \bm{Y}. 
\end{equation}
%Although there are strong theoretical guarantees for solving NNM problems using singular value thresholding, many studies have shown that such solutions often reach sub-optimality. This is because NNM imposes the same threshold on all singular values, contrary to the fact that larger singular values, which contain more information, should be penalized less. 
NNM provides strong theoretical guarantees and can be efficiently solved using singular value thresholding. However, in practice, its solutions are frequently suboptimal. This limitation arises because NNM applies an equal penalty to all singular values, disregarding the fact that larger singular values, which carry more significant information, should be penalized less.

%To further improve the NNM, many nonconvex functions have been introduced, e.g., truncated nuclear norms \cite{geng2018truncated}, weighted nuclear norms \cite{gu2014weighted,gu2017weighted}, Schatten $p$-norms \cite{lu2015nonconvex}. The most representative is the weighted nuclear norm minimization model (WNNM) proposed by Gu et al. \cite{gu2014weighted,gu2017weighted}. This model applies different weights to the singular values to ensure that the larger singular values undergo smaller shrinkage. Inspired by \cite{gu2014weighted,gu2017weighted}, \cite{xie2016weighted} proposed a weighted Schatten $p$-norms minimization (WSNM) by reweighting the Schatten $p$-norms.

To improve NNM, several nonconvex approaches have been proposed, such as truncated nuclear norms \cite{geng2018truncated}, weighted nuclear norms \cite{gu2014weighted,gu2017weighted}, and Schatten $p$-norms \cite{lu2015nonconvex}. Among these, the weighted nuclear norm minimization (WNNM) model \cite{gu2014weighted,gu2017weighted} is particularly noteworthy. WNNM assigns different weights to singular values, allowing larger singular values to undergo less shrinkage. Building on this concept, Xie et al. \cite{xie2016weighted} extended WNNM to the weighted Schatten $p$-norm minimization (WSNM) model, further refining the approximation by reweighting the Schatten $p$-norms.

With the increasing availability of multidimensional data, methods for effectively characterizing such data have garnered significant attention. Quaternions, as an extension of complex numbers, are widely employed for high-dimensional data such as color images \cite{jia2019color} and remote sensing images \cite{pan2023separable}, owing to their ability to efficiently represent multidimensional information. Recent advancements in quaternion low-rank matrix reconstruction have enhanced both theoretical understanding and practical applications. Theoretically, Jia et al. \cite{jia2019robust} demonstrated that a quaternion matrix can be exactly recovered with high probability under incoherence conditions, provided its rank is sufficiently low and the corrupted entries are sparse. This result was further extended to incorporate the nonlocal self-similarity (NSS) prior in images \cite{jia2022non}. Chen et al. \cite{chen2022color} proposed a novel minimization problem combining a nuclear norm with a three-channel weighted quadratic loss, offering error bounds for both clean and corrupted regions. As many quaternion problems ultimately involve solving quaternion linear systems, Jia et al. \cite{jia2021structure} developed the quaternion generalized minimum residual method (QGMRES) for efficiently solving such systems.

In practical applications, quaternion low-rank matrix recovery is extensively utilized in color image processing \cite{miao2021color,miao2020quaternion}. For example, quaternion weighted nuclear norm minimization (QWNNM), a direct extension of WNNM to quaternions, has been applied to tasks such as color image denoising and deblurring \cite{yu2019quaternion,huang2022quaternion}. Building on these works, Zhang et al. \cite{zhang2024quaternion} generalized WSNM to quaternions, proposing quaternion weighted Schatten $p$-norm minimization (QWSNM) for color image processing. Additionally, Chen et al. \cite{chen2019low} introduced several non-convex approximation models to better estimate singular values in noisy observations for low-rank quaternion matrices. Guo et al. \cite{guo2025quaternion} further advanced the field by proposing the quaternion nuclear norm minus Frobenius norm (QNMF), a novel nonconvex hybrid norm for color image reconstruction, achieving excellent performance across various color image tasks.

Recently, the $L_1$ norm-to-$L_2$ norm ratio ($L_1/L_2$) has gained significant attention for its superior approximation of $L_0$. The $L_1/L_2$ model shares $L_0$’s scale invariance and is parameter-free, distinguishing it from $L_p$ $(0<p<1)$ and $L_1-\alpha L_2$ $(\alpha>0)$ models. This concept originates from early works by Hoyer \cite{hoyer2002non} and Hurley and Rickard \cite{hurley2009comparing}. In recent studies, Yin et al. \cite{yin2014ratio} established the equivalence of $L_1/L_2$ and $L_0$ under nonnegative signals. Rahimi et al. \cite{rahimi2019scale} introduced the strong null space property (sNSP) and demonstrated that any sparse solution satisfies local optimality in the $L_1/L_2$ model under this condition. They also developed an alternating direction multiplier method (ADMM)-based algorithm for solving $L_1/L_2$. Subsequently, Wang et al. \cite{wang2021limited,wang2022minimizing} extended these findings to $L_1/L_2$ minimization over gradients in imaging applications. Tao \cite{tao2022minimization} derived the analytic solution for the proximal operator of $L_1/L_2$, proposed an ADMM-based splitting format, and proved global convergence under mild assumptions with a linear convergence rate under appropriate conditions. These contributions were further generalized in later works \cite{tao2023study,tao2024partly}. Additionally, $L_1/L_2$ has seen applications in sparse signal reconstruction \cite{wang2020accelerated,zeng2021analysis,ge2023analysis,xu2021analysis}, image segmentation \cite{wu2022efficient}, and image recovery \cite{chowdhury2024poissonian}.

To address the challenge of minimizing the rank of a quaternion matrix, it is crucial to identify a tighter and more practical relaxation for the quaternion rank minimization problem. In this work, we propose the quaternion nuclear norm over Frobenius norm (QNOF) minimization, inspired by the $L_1/L_2$ model, to approximate the rank of a quaternion matrix. The QNOF is both parameter-free and scale-invariant, consistent with the rank function property. 
Using quaternion singular value decomposition (QSVD), we prove that solving QNOF can be transformed into solving the $L_1/L_2$ singular value problem. This leads to a double simplification of the solution problem, i.e. from a matrix problem to a vector problem and from a quaternion space to a real space. We extend its application to matrix completion (MC), robust principal component analysis (RPCA), and robust matrix completion (RMC). By leveraging the alternating direction multiplier method (ADMM), we ensure convergence of the sequence under mild conditions. 
The most relevant prior works, \cite{gao2024low} and \cite{zheng2024scale}, generalize \cite{rahimi2019scale} to matrices and tensors. In contrast, our approach directly tackles the quaternion matrix problem. Instead of adopting the splitting method used in \cite{rahimi2019scale}, We transform the QNOF problem into a $L_1/L_2$ singular value minimization problem and solve it.

The main contributions of this paper are summarized as follows
\begin{itemize}
    \item We propose a new nonconvex quaternion low rank regularization, QNOF, which is parameter-free and scale-invariant compared to other popular nonconvex approximations of the rank function.

    \item To solve QNOF, we use QSVD and prove that solving QNOF can be transformed into solving the $L_1/L_2$ singular value minimization problem. This leads to a double simplification from solving the quaternion matrix problem to solving a real vector problem.

    \item We extend QNOF to matrix completion, robust principal component analysis, and robust matrix completion models. We use the ADMM scheme to solve these proposed models. In this scheme, the convergence of the whole sequence can be established by mild assumptions.

    % \item We conducted extensive experiments on real color images. The experimental results show that our proposed method can outperform several state-of-the-art convex or nonconvex algorithms.

\end{itemize}

% The rest of the paper is organized as follows. 
The relevant organization of this paper is as follows. Some foundations of quaternion algebra which are used in this paper are presented in \cref{sec:qab}. In \cref{sec:nof}, the QNOF model is presented and a solution scheme is given. Based on the QNOF model, in \cref{sec:alg}, we provide its generalization to matrix completion, robust principal component analysis, and robust matrix completion, and a proof of convergence is given. Numerical results are shown in \cref{sec:experiments} and conclusions are presented in \cref{sec:conclusions}.

\section{Quaternion algebraic basics}
\label{sec:qab}
The quaternion space $\mathbb{Q}$ is a generalization of both the real space $\mathbb{R}$ and the complex space $\mathbb{C}$, defined as: $\mathbb{ Q } = \{a_0 + a_1\bm{i} + a_2\bm{j} +a_3\bm{k}|a_0,a_1,a_2,a_3 \in \mathbb{R}\}$, where $\{1, \bm{i}, \bm{j}, \bm{k}\}$ forms a basis for $\mathbb{Q}$. The imaginary units $\bm{i}, \bm{j}, \bm{k}$ satisfy the following relations: 
$$\bm{i}^2=\bm{j}^2=\bm{k}^2=\bm{i}\bm{j}\bm{k}=-1,$$ 
$$\bm{i}\bm{j}=\bm{k}=-\bm{j}\bm{i}, \, \bm{j}\bm{k}=\bm{i}=-\bm{k}\bm{j},\, \mathrm{and}\, \bm{k}\bm{i}=\bm{j}=-\bm{i}\bm{k}.$$ 
One of the key properties of quaternion space $\mathbb{Q}$ is that quaternion multiplication is not commutative. For example, a quaternion variable is often denoted by placing a dot above the variable, such as $\dot{\bm{a}}$.

Let $\dot{\bm{a}} = a_0 + a_1\bm{i} + a_2\bm{j} + a_3\bm{k} \in \mathbb{Q}$, $\dot{\bm{b}} = b_0 + b_1\bm{i} + b_2\bm{j} + b_3\bm{k} \in \mathbb{Q}$, and $\lambda \in \mathbb{R}$. Their operations are as follows:
\begin{equation*}
\begin{aligned}
&\dot{\bm{a}} + \dot{\bm{b}} =   (a_0 + b_0) + (a_1+b_1)\bm{i} + (a_2+b_2)\bm{j} + (a_3+b_3)\bm{k}, \\
&\lambda\dot{\bm{a}} =  (\lambda a_0) + (\lambda a_1)\bm{i} + (\lambda a_2)\bm{j} + (\lambda a_3)\bm{k}, \\
&\dot{\bm{a}}\dot{\bm{b}} = (a_0b_0-a_1b_1-a_2b_2-a_3b_3) + (a_0b_1+a_1b_0+a_2b_3-a_3b_2)\bm{i} \\
&~~+ (a_0b_2-a_1b_3+a_2b_0+a_3b_1)\bm{j}  + (a_0b_3+a_1b_2-a_2b_1+a_3b_0)\bm{k}.
\end{aligned}
\end{equation*}
The conjugate and modulus of $\dot{\bm{a}}$ are defined as:
\begin{equation}
\begin{aligned}
\dot{\bm{a}}^{*} &= a_0 - a_1\bm{i} - a_2\bm{j} - a_3\bm{k}, \\
|\dot{\bm{a}}| &= \sqrt{a_{0}^{2}+a_{1}^{2}+a_{2}^{2}+a_{3}^{2}}.
\end{aligned}
\end{equation}
If $a_0 = 0$, $\dot{\bm{a}}$ is called a pure quaternion. Each quaternion $\dot{\bm{a}}$ can be uniquely represented as:
$$\dot{\bm{a}} = a_0 + a_1\bm{i} + (a_2 +a_3\bm{i})\bm{j} = c_{1} + c_{2}\bm{j},$$
where $c_{1}=a_0 + a_1\bm{i}$ and $c_{2}=a_2 +a_3\bm{i}$ are complex numbers.

For a quaternion matrix defined as $\dot{\bm{X}} = (\dot{x}_{ij}) \in \mathbb{ Q }^{m\times n}$, where $\dot{\bm{X}} = \bm{X}_0 + \bm{X}_1\bm{i} + \bm{X}_2\bm{j}+ \bm{X}_3\bm{k}$ and $\bm{X}_l\in \mathbb{R}^{m\times n}(l=0,1,2,3)$ are real matrices. The modulus of a quaternion matrix is $|\dot{\bm{X}}| = (|\dot{x}_{ij}|) \in \mathbb{ R }^{m\times n}$. For a color image, it can be represented as a pure quaternion matrix with the real part set to zero, i.e., each channel is represented as an imaginary part:
\begin{equation}
\begin{aligned}
\dot{\bm{X}}_{RGB} =  \bm{X}_{R}\bm{i} + \bm{X}_{G}\bm{j}+ \bm{X}_{B}\bm{k},
\end{aligned}
\end{equation}
where $\bm{X}_R$, $\bm{X}_G$, and $\bm{X}_B$ are the red, green, and blue channels of the color image.

The norms of quaternion vectors and matrices are defined as follows. The $l_1$, $l_2$, and $\infty$ norms of a quaternion vector are given by
$$
||\dot{\bm{a}}||_1:=\sum_{i}^{n}|a_i|,  \quad ||\dot{\bm{a}}||_2:=\sqrt{\sum_{i}^{n}|a_i|^2} \quad \mathrm{and}  \quad||\dot{\bm{a}}||_{\infty}:= \max_{1\leq i \leq n} |a_i|,
$$
respectively.
For a quaternion matrix, the $l_1$ and $\infty$ norms are defined as
$$
||\dot{\bm{X}}||_{1} := \sum_{i,j}| \dot{x}_{ij}| \quad \mathrm{and} \quad ||\dot{\bm{X}}||_{\infty} := \max_{i,j}| \dot{x}_{ij}|.
$$
The Frobenius norm of a quaternion matrix is given by
$$
||\dot{\bm{X}}||_{F} := \sqrt{\sum_{i,j}| \dot{x}_{ij}|^{2}} =  \sqrt{\mathrm{Tr}(\dot{\bm{X}}^{*}\dot{\bm{X}})} = \sqrt{\mathrm{Tr}(\dot{\bm{X}}\dot{\bm{X}}^{*})} = \sqrt{\sum_{k}^{\min{\{i,j\}}}  \sigma_{\dot{\bm{X}},k}^{2} },
$$
where $\sigma_{\dot{\bm{X}},k}$ denotes the $k$-th singular value and $\dot{\bm{X}}^{*}$ denotes the conjugate transpose of $\dot{\bm{X}}$. 
The nuclear norm of a quaternion matrix is defined as
$$
||\dot{\bm{X}}||_{*} := \sum_{k}^{\min{\{i,j\}}}  \sigma_{\dot{\bm{X}},k}.
$$
Finally, the singular value decomposition (SVD) of a quaternion matrix $\dot{\bm{X}}$ is presented in the following theorem.

\begin{theorem}[QSVD \cite{zhang1997quaternions}]
Given a quaternion matrix $\dot{\bm{X}} \in \mathbb{ Q }^{m\times n}$ with rank $r$.  There are two unitary quaternion matrices $\dot{\bm{U}} \in \mathbb{ Q }^{m\times m}$ and $\dot{\bm{V}} \in \mathbb{ Q }^{n\times n}$ satisfying $\dot{\bm{X}}  = \dot{\bm{U}}\left(\begin{array}{cc} \bm{\Sigma_{r}} & 0\\ 0&0 \end{array}\right)\dot{\bm{V}}^{*}$, where $\bm{\Sigma_{r}} = \mathrm{diag}(\sigma_1,\dots,\sigma_r)\in \mathbb{R}^{r\times r}$, and all singular values $\sigma_i>0, i=1,\dots,r$.
\label{qsvd}
\end{theorem}

\section{Quaternion nuclear norm over Frobenius norm minimization}
\label{sec:nof}

\subsection{The proposed QNOF model}
The quaternion nuclear norm over the Frobenius norm (QNOF) is defined through the following optimization model:
\begin{equation}
\begin{aligned}
\min_{\dot{\bm{X}}} \frac{1}{2}|| \dot{\bm{Y}}-\dot{\bm{X}} ||^{2}_F + \lambda\left(\frac{||\dot{\bm{X}}||_{*}}{||\dot{\bm{X}}||_{F}} \right), \\ 
\end{aligned}
\label{QNOF}
\end{equation}
where $\lambda$ is a regularization parameter. For convenience, we define:
\begin{equation}
||\dot{\bm{X}} ||_{\mathrm{QNOF}} = \frac{||\dot{\bm{X}}||_{*}}{||\dot{\bm{X}}||_{F}}.
\end{equation}
We next discuss scaling and unitary transformation invariants and boundedness of $||\cdot||_{\mathrm{QNOF}}$.
%%%%%

\begin{proposition}[scale invariance] Let $\dot{\bm{A}} \in \mathbb{ Q }^{m\times n}$, we have
$$ ||\dot{\bm{A}}||_{\mathrm{QNOF}} = ||c\dot{\bm{A}}||_{\mathrm{QNOF}},$$
which holds for any nonzero scalar $c$.
\label{p3.1}
\end{proposition}
\begin{proof}
    To maintain generality, we assume $m \geq n$. The singular values of $\dot{\bm{A}}$ are expressed as
    $$
    \bm{\sigma} = [\sigma_1, \sigma_2, \dots, \sigma_n] \in \mathbb{R}^{n}. 
    $$
Using the definitions of the nuclear and Frobenius norms, we have:
    \begin{equation}
        ||c\dot{\bm{A}}||_{\mathrm{QNOF}} = \frac{||c\bm{\sigma}||_1}{||c\bm{\sigma}||_2} = \frac{||\bm{\sigma}||_1}{||\bm{\sigma}||_2} = ||\dot{\bm{A}}||_{\mathrm{QNOF}}.
    \end{equation}
\end{proof}
% \begin{proposition}[unitary invariance] For any $\dot{\bm{A}} \in \mathbb{Q}^{m \times n}$ and orthogonal quaternion matrices $\dot{\bm{P}}$ and $\dot{\bm{Q}}$,
% $$
% ||\dot{\bm{A}}||_{\mathrm{QNOF}} = ||\dot{\bm{P}}\dot{\bm{A}}||_{\mathrm{QNOF}} = ||\dot{\bm{A}}\dot{\bm{Q}}^{*}||_{\mathrm{QNOF}} = ||\dot{\bm{P}}\dot{\bm{A}}\dot{\bm{Q}}^{*}||_{\mathrm{QNOF}}.
% $$
% \label{p3.2}
% \end{proposition}
% \begin{proof}
%     We prove only $||\dot{\bm{P}}\dot{\bm{A}}||_{\mathrm{QNOF}}=||\dot{\bm{A}}||_{\mathrm{QNOF}}$. The rest of the cases are similar.
%     \begin{equation}
%         ||\dot{\bm{P}}\dot{\bm{A}}||_{\mathrm{QNOF}} = \frac{||(\dot{\bm{P}}\dot{\bm{U}})\dot{\bm{\Sigma}}\dot{\bm{V}}^{*}||_{*}}{||(\dot{\bm{P}}\dot{\bm{U}})\dot{\bm{\Sigma}}\dot{\bm{V}}^{*}||_{F}} = \frac{||\bm{\sigma}||_1}{||\bm{\sigma}||_2}  = ||\dot{\bm{A}}||_{\mathrm{QNOF}},
%     \end{equation}
%     where $\dot{\bm{A}} = \dot{\bm{U}}\dot{\bm{\Sigma}}\dot{\bm{V}}^{*}$ is obtained from QSVD.
% \end{proof}
\begin{proposition}[unitary invariance] Let $\dot{\bm{A}} \in \mathbb{ Q }^{m\times n}$, QNOF satisfies 
$$
||\dot{\bm{A}}||_{\mathrm{QNOF}} = ||\dot{\bm{P}}\dot{\bm{A}}||_{\mathrm{QNOF}} = ||\dot{\bm{A}}\dot{\bm{Q}}^{*}||_{\mathrm{QNOF}} = ||\dot{\bm{P}}\dot{\bm{A}}\dot{\bm{Q}}^{*}||_{\mathrm{QNOF}},
$$
where $\dot{\bm{P}}$ and $\dot{\bm{Q}}$ are any orthogonal quaternion matrices.
\label{p3.2}
\end{proposition}
\begin{proof}
    We prove only $||\dot{\bm{P}}\dot{\bm{A}}||_{\mathrm{QNOF}}=||\dot{\bm{A}}||_{\mathrm{QNOF}}$. The rest of the cases are similar.
    \begin{equation}
        ||\dot{\bm{P}}\dot{\bm{A}}||_{\mathrm{QNOF}} = \frac{||(\dot{\bm{P}}\dot{\bm{U}})\dot{\bm{\Sigma}}\dot{\bm{V}}^{*}||_{*}}{||(\dot{\bm{P}}\dot{\bm{U}})\dot{\bm{\Sigma}}\dot{\bm{V}}^{*}||_{F}} = \frac{||\bm{\sigma}||_1}{||\bm{\sigma}||_2}  = ||\dot{\bm{A}}||_{\mathrm{QNOF}},
    \end{equation}
    where $\dot{\bm{A}} = \dot{\bm{U}}\dot{\bm{\Sigma}}\dot{\bm{V}}^{*}$ is obtained from QSVD.
\end{proof}
\begin{proposition}[boundedness] For any nonzero $\dot{\bm{A}} \in \mathbb{ Q }^{m\times n}$, we have
$$
1 \leq ||\dot{\bm{A}}||_{\mathrm{QNOF}} \leq \sqrt{\mathrm{rank}(\dot{\bm{A}})} \leq \min\{\sqrt{m}, \sqrt{n}\}.
$$
\label{p3.3}
\end{proposition}
\begin{proof}
%By $||\bm{x}||_2 \leq ||\bm{x}||_1 \leq \sqrt{||\bm{x}||_0}||\bm{x}||_2 $ for any vector $\bm{x}$, we obtain
For any vector $\bm{x}$, the inequalities $|\bm{x}|_2 \leq |\bm{x}|_1 \leq \sqrt{|\bm{x}|_0}|\bm{x}|_2$ yield:
\begin{equation}
    1 \leq \frac{||\bm{\sigma}||_1}{||\bm{\sigma}||_2} \leq \sqrt{||\bm{\sigma}||_0}.
\end{equation}
Consequently,
\begin{equation*}
 1 \leq ||\dot{\bm{A}}||_{\mathrm{QNOF}} \leq \sqrt{\mathrm{rank}(\dot{\bm{A}})} \leq \min\{\sqrt{m}, \sqrt{n}\}.
\end{equation*}
\end{proof}

From Properties \ref{p3.1}-\ref{p3.3}, we observe that $||\cdot||_{\mathrm{QNOF}}$ exhibits scale invariance and unitary invariance, aligning closely with the properties of the rank function. Thus, $||\cdot||_{\mathrm{QNOF}}$  serves as a superior nonconvex approximation of the rank function compared to other norms like the nuclear norm.
To solve \eqref{QNOF}, we introduce a generalized version of Von Neumann’s trace inequality for quaternion matrices \cite{lewis1995convex,mirsky1975trace}.

\begin{lemma}
For $\dot{\bm{A}}, \dot{\bm{B}} \in \mathbb{Q}^{m \times n}$ with $m \geq n$, the following inequality holds: 
\begin{equation}
\mathrm{Re}(\mathrm{Tr}(\dot{\bm{A}}^{*}\dot{\bm{B}})) \leq \bm{\Sigma_{\dot{A}}}^{\top}\bm{\Sigma_{\dot{B}}},
\label{v-n}
\end{equation}
where $ \sigma_{\dot{\bm{A}},1} \geq \cdots \geq \sigma_{\dot{\bm{A}},n} \geq 0$ and
$ \sigma_{\dot{\bm{B}},1} \geq \cdots \geq \sigma_{\dot{\bm{B}},n} \geq 0 $ are the descending singular values of $\dot{\bm{A}}$ and $\dot{\bm{B}}$, respectively. The case of equality occurs if and only if it is possible to find two unitary $\dot{\bm{U}} \in \mathbb{ Q }^{m\times m}$ and $\dot{\bm{V}} \in \mathbb{ Q }^{n\times n}$ that simultaneously singular value decompose $\dot{\bm{A}}$ and $\dot{\bm{B}}$ in the sense that
\begin{equation}
\dot{\bm{A}} = \dot{\bm{U}} \bm{\Sigma_{\dot{A}}} \dot{\bm{V}}^{*} \quad and \quad
\dot{\bm{B}} = \dot{\bm{U}} \bm{\Sigma_{\dot{B}}} \dot{\bm{V}}^{*},
\end{equation}
where $\bm{\Sigma_{\dot{A}}} = \diag(\sigma_{\dot{\bm{A}},1},  \cdots,\sigma_{\dot{\bm{A}},n} )$ and $\bm{\Sigma_{\dot{B}}} = \diag(\lambda_{\bm{B},1},  \cdots, \lambda_{\bm{B},n})$.
\label{V-N T}
\end{lemma}

\begin{proof}
By Theorem \ref{qsvd}, there exist unitary matrices $\dot{\bm{U}}_i \in \mathbb{Q}^{m \times m}$ and $\dot{\bm{V}}_i \in \mathbb{Q}^{n \times n}$ $(i = 1, 2)$ such that 
\begin{equation}
\dot{\bm{A}} = \dot{\bm{U}}_1 \bm{\Sigma_{\dot{A}}} \dot{\bm{V}}_1^{*} \quad and \quad
\dot{\bm{B}} = \dot{\bm{U}}_2 \bm{\Sigma_{\dot{B}}} \dot{\bm{V}}_2^{*},
\end{equation}
where $\bm{\Sigma_{\dot{A}}} = \diag(\sigma_{\dot{\bm{A}},1},  \cdots,\sigma_{\dot{\bm{A}},n} )$ and $\bm{\Sigma_{\dot{B}}} = \diag(\lambda_{\bm{B},1},  \cdots, \lambda_{\bm{B},n})$ are the singular value matrices of $\dot{\bm{A}}$ and $\dot{\bm{B}}$, respectively.

For quaternion matrices, $\mathrm{Tr}(\dot{\bm{A}}^{*}\dot{\bm{B}}) = \mathrm{Tr}(\dot{\bm{B}}^{*}\dot{\bm{A}})$ does not generally hold due to the non-commutativity of quaternion multiplication. 
However, the real part satisfies $\mathrm{Re}(\mathrm{Tr}(\dot{\bm{A}}^{}\dot{\bm{B}})) = \mathrm{Re}(\mathrm{Tr}(\dot{\bm{B}}^{}\dot{\bm{A}}))$. Using this, we
\begin{equation}
\begin{aligned}
\mathrm{Re}(\mathrm{Tr}(\dot{\bm{A}}^{*}\dot{\bm{B}})) &= \mathrm{Re}(\mathrm{Tr}(\dot{\bm{V}}_1 \bm{\Sigma_{\dot{A}}}^{\top} \dot{\bm{U}}_1^{*} \dot{\bm{U}}_2 \bm{\Sigma_{\dot{B}}} \dot{\bm{V}}_2^{*})) \\
&= \mathrm{Re}(\mathrm{Tr}(\dot{\bm{V}}_2^{*} \dot{\bm{V}}_1 \bm{\Sigma_{\dot{A}}}^{\top} \dot{\bm{U}}_1^{*} \dot{\bm{U}}_2 \bm{\Sigma_{\dot{B}}} )) \\
&= \mathrm{Re}(\mathrm{Tr}(\dot{\bm{V}} \bm{\Sigma_{\dot{A}}}^{\top} \dot{\bm{U}} \bm{\Sigma_{\dot{B}}} )) \\
&= \mathrm{Re}\left( \sum_{r,s=1}^{n} \dot{v}_{r,s}\dot{u}_{s,r}\sigma_{\dot{\bm{A}},s}\sigma_{\dot{\bm{B}},r} \right), 
\end{aligned}
\end{equation}
where $\dot{\bm{V}} = (\dot{v}_{r,s})_{n\times n} = \dot{\bm{V}}_2^{*} \dot{\bm{V}}_1$ and $\dot{\bm{U}} = (\dot{u}_{r,s})_{m\times m} = \dot{\bm{U}}_1^{*} \dot{\bm{U}}_2$ are two unitary matrices. Hence, 
\begin{equation}
\begin{aligned}
\mathrm{Re}\left( \sum_{r,s=1}^{n} \dot{v}_{r,s}\dot{u}_{s,r}\sigma_{\dot{\bm{A}},s}\sigma_{\dot{\bm{B}},r} \right) 
&\leq \left|\mathrm{Re}\left( \sum_{r,s=1}^{n} \dot{v}_{r,s}\dot{u}_{s,r}\sigma_{\dot{\bm{A}},s}\sigma_{\dot{\bm{B}},r} \right)\right|  \\
&\leq \sum_{r,s=1}^{n} |\dot{v}_{r,s}\dot{u}_{s,r}|\sigma_{\dot{\bm{A}},s}\sigma_{\dot{\bm{B}},r} \\
&\stackrel{(i)}{\leq} \frac{1}{2}\sum_{r,s=1}^{n} |\dot{v}_{r,s}|^{2}\sigma_{\dot{\bm{A}},s}\sigma_{\dot{\bm{B}},r} + \frac{1}{2}\sum_{r,s=1}^{n} |\dot{u}_{s,r}|^{2}\sigma_{\dot{\bm{A}},s}\sigma_{\dot{\bm{B}},r} \\
&\stackrel{(ii)}{\leq} \sum_{r=1}^{n}\sigma_{\dot{\bm{A}},r}\sigma_{\dot{\bm{B}},r}   = \bm{\Sigma_{\dot{A}}}^{\top}\bm{\Sigma_{\dot{B}}}.
\end{aligned}
\end{equation}
Here, inequality (i) follows from the Cauchy-Schwarz inequality. For inequality (ii), see \cite[Lemma]{mirsky1975trace}. Clearly, the equation in \eqref{v-n} holds if and only if $\dot{\bm{U}}_1 = \dot{\bm{U}}_2$ and $\dot{\bm{V}}_1 = \dot{\bm{V}}_2$.
\end{proof}

% 基于引理\ref{V-N T},可以得到如下定理。
Based on Lemma \ref{V-N T}, the following crucial theorem can be derived.
\begin{theorem}
Let $\dot{\bm{Y}} \in \mathbb{Q}^{m \times n}$, and without loss of generality, assume $m \geq n$. Let $\dot{\bm{Y}} = \dot{\bm{U}} \bm{\Sigma_{\dot{Y}}} \dot{\bm{V}}^{*}$ denote the QSVD of $\dot{\bm{Y}}$, where $\bm{\Sigma_{\dot{Y}}} = \mathrm{diag}(\sigma_{\dot{\bm{Y}},1}, \sigma_{\dot{\bm{Y}},2}, \dots, \sigma_{\dot{\bm{Y}},n})$.
The global optimum of model \eqref{QNOF} is given by $\dot{\bm{X}}=\dot{\bm{U}} \bm{\tilde{\Sigma}_{\dot{X}}} \dot{\bm{V}}^{*}$, where $\bm{\tilde{\Sigma}_{\dot{X}}}=\mathrm{diag}(\tilde{\sigma}_{\dot{\bm{X}},1},$ $ \tilde{\sigma}_{\dot{\bm{X}},2},\dots,\tilde{\sigma}_{\dot{\bm{X}},n})$ is a diagonal nonnegative matrix. The matrix $\bm{\tilde{\Sigma}_{\dot{X}}}$ is the solution to the following optimization problem:
\begin{equation}
\begin{aligned}
\bm{\tilde{\Sigma}_{\dot{X}}} = \arg&\min_{\bm{\Sigma_{\dot{X}}}}  \frac{1}{2} ||\bm{\Sigma_{\dot{Y}}} - \bm{\Sigma_{\dot{X}}} ||^{2}_{2} + \lambda\left(\frac{||\bm{\Sigma_{\dot{X}}}||_{1}}{||\bm{\Sigma_{\dot{X}}}||_{2}}\right) \\
& \mathrm{s.t.} \sigma_{\dot{\bm{X}},1} \geq \sigma_{\dot{\bm{X}},2} \geq \dots \geq \sigma_{\dot{\bm{X}},n} \geq 0. 
\end{aligned}
\end{equation}
\label{theorem01}
\end{theorem}

\begin{proof}[Proof]
Let $\dot{\bm{X}} = \hat{\dot{\bm{U}}} \bm{\Sigma_{\dot{X}}} \hat{\dot{\bm{V}}}^{}$ and $\dot{\bm{Y}} = \dot{\bm{U}} \bm{\Sigma_{\dot{Y}}} \dot{\bm{V}}^{}$ represent the QSVD decompositions of $\dot{\bm{X}}$ and $\dot{\bm{Y}}$, respectively. We proceed as follows:
\begin{equation}
\begin{aligned}
& \frac{1}{2}|| \dot{\bm{Y}}-\dot{\bm{X}} ||^{2}_F + \lambda\left(\frac{||\dot{\bm{X}}||_{*}}{||\dot{\bm{X}}||_{F}} \right) \\
&= \frac{1}{2}\mathrm{Tr}(\dot{\bm{Y}}^{*}\dot{\bm{Y}}) + \frac{1}{2}\mathrm{Tr}(\dot{\bm{X}}^{*}\dot{\bm{X}}) -  \mathrm{Re}(\mathrm{Tr}(\dot{\bm{Y}}^{*}\dot{\bm{X}})) + \lambda\left(\frac{||\dot{\bm{X}}||_{*}}{||\dot{\bm{X}}||_{F}} \right) \\
&= \frac{1}{2}|| \dot{\bm{Y}}||^{2}_F + \frac{1}{2}|| \dot{\bm{X}}||^{2}_F - \mathrm{Re}(\mathrm{Tr}(\dot{\bm{Y}}^{*}\dot{\bm{X}})) + \lambda\left(\frac{||\dot{\bm{X}}||_{*}}{||\dot{\bm{X}}||_{F}} \right) \\
&= \frac{1}{2}( || \bm{\Sigma_{\dot{Y}}} ||^{2}_{2} + || \bm{\Sigma_{\dot{X}}} ||^{2}_{2}) - \mathrm{Re}(\mathrm{Tr}(\dot{\bm{Y}}^{*}\dot{\bm{X}})) + \lambda\left(\frac{|| \bm{\Sigma_{\dot{X}}}||_{1}}{||\bm{\Sigma_{\dot{X}}}||_{2}}\right) \\
&\geq \frac{1}{2}( || \bm{\Sigma_{\dot{Y}}} ||^{2}_{2} + || \bm{\Sigma_{\dot{X}}} ||^{2}_{2}) - \bm{\Sigma_{\dot{Y}}}^{\top}\bm{\Sigma_{\dot{X}}} + \lambda\left(\frac{|| \bm{\Sigma_{\dot{X}}}||_{1}}{||\bm{\Sigma_{\dot{X}}}||_{2}}\right). \\
& = \frac{1}{2} ||\bm{\Sigma_{\dot{Y}}} - \bm{\Sigma_{\dot{X}}} ||^{2}_{2} + \lambda\left(\frac{||\bm{\Sigma_{\dot{X}}}||_{1}}{||\bm{\Sigma_{\dot{X}}}||_{2}}\right).
\end{aligned}
\end{equation}
By Lemma \ref{V-N T} (Von Neumann's trace inequality), we have $\mathrm{Re}(\mathrm{Tr}(\dot{\bm{Y}}^{*}\dot{\bm{X}})) \leq \bm{\Sigma_{\dot{Y}}}^{\top}\bm{\Sigma_{\dot{X}}}$, with equality holding if and only if $\dot{\bm{U}} = \hat{\dot{\bm{U}}}$ and $\dot{\bm{V}} = \hat{\dot{\bm{V}}}$. Thus, solving model \eqref{QNOF} reduces to solving for$\bm{\tilde{\Sigma}_{\dot{X}}}$ in the following optimization problem:
\begin{equation}
\begin{aligned}
\bm{\tilde{\Sigma}_{\dot{X}}} = \arg&\min_{\bm{\Sigma_{\dot{X}}}} \frac{1}{2} ||\bm{\Sigma_{\dot{Y}}} - \bm{\Sigma_{\dot{X}}} ||^{2}_{2} + \lambda\left(\frac{||\bm{\Sigma_{\dot{X}}}||_{1}}{||\bm{\Sigma_{\dot{X}}}||_{2}}\right) \\
& \mathrm{s.t.} ~ \sigma_{\dot{\bm{X}},1} \geq \sigma_{\dot{\bm{X}},2} \geq \dots \geq \sigma_{\dot{\bm{X}},n} \geq 0,
\end{aligned}
\label{l1-l2}
\end{equation}
where $\bm{\tilde{\Sigma}_{\dot{X}}}=\mathrm{diag}(\tilde{\sigma}_{\dot{\bm{X}},1},\tilde{\sigma}_{\dot{\bm{X}},2},\dots,\tilde{\sigma}_{\dot{\bm{X}},n})$. 
\end{proof}

%Theorem \ref{theorem01} shows that the NOF minimization problem for quaternion matrices \eqref{QNOF} can be transformed into a singular-valued $L_1/L_2$ problem \eqref{l1-l2}. This transformation offers a twofold advantage: it converts the problem from the quaternion space to the real number space and simplifies it from a matrix problem to a vector problem

Theorem \ref{theorem01} demonstrates that the NOF minimization problem for quaternion matrices \eqref{QNOF} can be reformulated as a singular-value-based $L_1/L_2$ optimization problem \eqref{l1-l2}. This reformulation provides two key advantages: it reduces the problem from quaternion space to real space and simplifies it further from a matrix-level problem to a vector-level problem.

\subsection{Solving for QNOF}
Recent works have achieved remarkable progress in solving the $L_1/L_2$ problem \cite{rahimi2019scale,wang2022minimizing,wang2021limited,tao2022minimization,tao2023study}. Building on the insights from \cite{tao2022minimization}, we propose the following theorem to address the singular value $L_1/L_2$ problem.

\begin{theorem}
Given that $ \bm{\Sigma_{\dot{Y}}} (\neq \bm{0}) \in \mathbb{ R }^{n}$ and $\lambda > 0$, for each optimal solution  $\bm{\tilde{\Sigma}_{\dot{X}}}$ $(t= ||\bm{\tilde{\Sigma}_{\dot{X}}} ||_0$, $a = || \bm{\tilde{\Sigma}_{\dot{X}}} ||_1$, r = $|| \bm{\tilde{\Sigma}_{\dot{X}}} ||_2)$, one of the following assertions holds:
\begin{itemize}
    \item[$\mathrm{(a)}$]  %If $0 < \frac{1}{\lambda} \leq a/r^3$, we denote an integer $s$ as the number of elements in $\bm{\Sigma_{\dot{Y}}}$ with the same largest magnitude, i.e., $\sigma_{\dot{\bm{Y}},1} = \dots = \sigma_{\dot{\bm{Y}},s} > \sigma_{\dot{\bm{Y}},s+1}$, then there exists one index $j \in [s]$ such that
    If $0 < \frac{1}{\lambda} \leq \frac{a}{r^3}$, let $s$ be the number of elements in $\bm{\Sigma_{\dot{Y}}}$ that have the largest magnitude, i.e., $\sigma_{\dot{\bm{Y}},1} = \dots = \sigma_{\dot{\bm{Y}},s} > \sigma_{\dot{\bm{Y}},s+1}$. Then, there exists an index $j \in [s]$ such that
    \begin{align} 
    \sigma_{\dot{\bm{X}},i} = \left\{
    \begin{aligned}
    &\sigma_{\dot{\bm{Y}},i},~~~i=j;  \\
    &0,~~~~~~otherwish.\\
    \end{aligned}\right.
    \label{2.5}
    \end{align}
    
    \item[$\mathrm{(b)}$]  If $ \frac{1}{\lambda} > a/r^3$, then the pair ($a, r$) satisfies ($Q^t = \sum^t_{i=1} | \sigma_{\dot{\bm{Y}},i}| $)
    \begin{subnumcases} {\label{2.6} }
    &$\dfrac{a^2}{r^3}- \dfrac{a}{\lambda} + \dfrac{Q^t}{\lambda} - \dfrac{t}{r} =0,$ \label{2.6a}\\
    &$r^3 - \left(\sum^{t}_{i=1}\sigma^{2}_{\dot{\bm{Y}},i}\right)r + \lambda (a- Q^t) = 0,$  \label{2.6b} 
    \end{subnumcases}
    and the $r$ is also satisfied with
    \begin{equation}
    \begin{aligned}
    \sigma_{\dot{\bm{Y}},t} > \frac{\lambda}{r} \quad and \quad  \sigma_{\dot{\bm{Y}},t+1} \leq \frac{\lambda}{r}
    \end{aligned}
    \label{2.7}
    \end{equation}
    and the solution $\bm{\tilde{\Sigma}_{\dot{X}}}$ is characterized by
    \begin{align} 
    \tilde{\sigma}_{\dot{\bm{X}},i} = \left\{
    \begin{aligned}
    &\frac{\dfrac{1}{\lambda}\sigma_{\dot{\bm{Y}},i}-\dfrac{1}{r}}{\dfrac{1}{\lambda}-\dfrac{a}{r^3}},~~~1 \leq i \leq t;  \\
    &0,~~~~~~otherwish.\\
    \end{aligned}\right.
    \label{L1-L2-S}
    \end{align}
\end{itemize}
\label{theorem001}
\end{theorem}

Since the singular value vector is a special nonnegative, descending vector, the proof of Theorem \ref{theorem001} closely follows the approach in \cite[Theorem 3.2]{tao2022minimization}, with detailed available in \cite{tao2022minimization}. Although Theorem \ref{theorem001} provides the optimal solution to \eqref{l1-l2}, this solution depends on the unknown parameters $a$ and $r$. While it may appear impractical to determine the optimal solution directly from Theorem \ref{theorem001}, it is sufficient to identify one optimal solution rather than all possible ones. By describing a specific solution based on the value of $\lambda$, we can effectively reduce the dependency on $a$ and $r$. The following theorem presents this result.

\begin{theorem}
Given $ \bm{\Sigma_{\dot{Y}}} (\neq \bm{0}) \in \mathbb{ R }^{n}$ and $\lambda > 0$, one of the following assertions holds:
\begin{itemize}
    \item[$\mathrm{(i)}$] If $0 < \frac{1}{\lambda} \leq \frac{1}{\sigma^{2}_{\dot{\bm{Y}},1}}$, then equation (\ref{l1-l2}) has a solution, as given by (\ref{2.5}). 

    \item[$\mathrm{(ii)}$] If $\frac{1}{\lambda} > \frac{1}{\sigma^{2}_{\dot{\bm{Y}},1}}$, then equation (\ref{l1-l2}) has a solution given by (\ref{L1-L2-S}), where the pair $(a, r)$ satisfies both (\ref{2.6}) and (\ref{2.7}) simultaneously.
\end{itemize}
\label{theorem0001}
\end{theorem}

The proof of Theorem \ref{theorem0001} is similar to that of \cite[Theorem 3.3]{tao2022minimization}. Compared to Theorem \ref{theorem001}, Theorem \ref{theorem0001} relaxes the dependence on $a$ and $r$ by changing the condition $a/r^3$ with $1/\sigma^{2}_{\dot{\bm{Y}},1}$. Additionally, \cite[Lemmas 3.5 and 3.6]{tao2022minimization} establish the uniqueness of the $(a, r)$ pairs when the sparsity $t$ is given. The solutions to the quadratic function \eqref{2.6a} for $a$ and the cubic function \eqref{2.6b} for $r$ are also provided. Therefore, when the sparsity $t$ is known, an algorithm for solving $(a, r)$ using alternating iterations can be written as follows:

\begin{subnumcases} {\label{l1l2-ar} }
\Delta^{(2)}_{l+1} &= $\dfrac{1}{\lambda^2} - \dfrac{4}{r^{3}_{l}}(\dfrac{Q^t}{\lambda} - \dfrac{t}{r_l}),$ \label{l1l2-ar1}\\
a_{l+1} &= $\dfrac{r^{3}_{l}}{2}(\dfrac{1}{\lambda} - \sqrt{\Delta^{(2)}_{l+1}}),$ \label{l1l2-ar2} \\
\phi_{l+1} &= $\arccos (\dfrac{\lambda(a_{l+1}-Q^t)}{2\varrho^{3}}),$  \label{l1l2-ar3} \\
r_{l+1} &= $2\varrho\cos (\pi/3-\phi_{l+1}/3),$ \label{l1l2-ar4}
\end{subnumcases}
where $\varrho=\sqrt{\sum^{t}_{i=1} \sigma^{2}_{\dot{\bm{Y}},i} /3}$, $l$ is the iteration number.

Next, consider the sparsity $t$. It can be observed that $a$, $r$, and \eqref{L1-L2-S} all depend on the true sparsity $t$. If $0 < \frac{1}{\lambda} \leq \frac{1}{\sigma^{2}{\dot{\bm{Y}},1}}$, we obtain a one-sparse solution using \eqref{2.5}. 
If $\frac{1}{\lambda} > \frac{1}{\sigma^{2}{\dot{\bm{Y}},1}}$, we iteratively search for the sparsity $t$ using the bisection method. In each iteration, we set the lower and upper bounds for $t$ as $t_1$ and $t_2$, respectively. We then compute the midpoint $\hat{t} = \frac{t_1 + t_2}{2}$ and update the corresponding $(a, r)$ pairs alternately according to \eqref{l1l2-ar}.

If the discriminant \eqref{l1l2-ar1} $\Delta^{(2)}_{l+1}<0$, we update $\hat{t}$ as follows: if $\sigma_{\dot{\bm{Y}},t+1} > \dfrac{\lambda}{r}$, we set $t_1 = \hat{t}$; otherwise, $t_2 = \hat{t}$.

If both conditions $(\sigma_{\dot{\bm{Y}},t} - \lambda/\hat{r})(1-\hat{a}\lambda/\hat{r}^{3})>0$ and $(\sigma_{\dot{\bm{Y}},t+1} - \lambda/\hat{r})(1-  \hat{a}\lambda/\hat{r}^{3}) \leq 0$ are satisfied, the solution $\bm{\tilde{\Sigma}_{\dot{X}}}$ is obtained through \eqref{L1-L2-S}. 

If conditions $(\sigma_{\dot{\bm{Y}},t} - \lambda/\hat{r})(1-\hat{a}\lambda/\hat{r}^{3})>0$ and $(\sigma_{\dot{\bm{Y}},t+1} - \lambda/\hat{r})(1- \hat{a}\lambda/\hat{r}^{3}) > 0$ are satisfied, we set $t_1 = \hat{t}$; otherwise, if  $(\sigma_{\dot{\bm{Y}},t} - \lambda/\hat{r})(1-\hat{a}\lambda/\hat{r}^{3}) \leq 0$, we set $t_2 = \hat{t}$. 

By bisecting the update $\hat{t} = (t_1+t_2)/2$ and alternating iterations $a$ and $r$, the unique solution $(a, r)$ satisfying \eqref{2.6} and \eqref{2.7}, along with the sparsity $t$, can eventually be found. In summary, the overall solution to the singular value $L_1/L_2$ problem \eqref{l1-l2} is shown in Algorithm \ref{alg:1}.

%%%%%% 二分法更新事项  %%%% 

\begin{algorithm}[!t]
	\caption{Solution to the singular value $L_1/L_2$ problem}
	\label{alg:1}
    \LinesNumbered 
		\KwIn{ $\lambda>0$, $ \bm{\Sigma_{\dot{Y}}} (\neq \bm{0}) \in \mathbb{ R }^{n}$, $l_{\mathrm{max}}$, $\epsilon$, $\omega$. }%\\
        % \textbf{Intialize}: $\bm{{\tilde{\Sigma}_{\dot{X}}}_{0}} = \max(\bm{\Sigma_{\dot{Y}}} - \omega,0)$}

        \eIf{$\frac{1}{\lambda} \leq \frac{1}{\sigma^{2}_{\dot{\bm{Y}},1}}$}{
         $\bm{\tilde{\Sigma}_{\dot{X}}}$ is computed by (\ref{2.5}) \tcp*[f]{one-sparse solution}}
         {
            Set $t_1=1$, $t_2=n$. \\
           \While(\tcp*[f]{bisection search $t$}){ $t_2 - t_1 >1$}{
                Set $t = (t_1+t_2)/2$, $r_0$, Flag1=0.\\
                \For(\tcp*[f]{split iteration solution $a,r$}){$l=0:l_{\mathrm{max}}$}{  % for 更新(a,r)
                    Compute $\Delta^{(2)}_{l+1}$ by (\ref{l1l2-ar1}). \\
                    \If{$\Delta^{(2)}_{l+1}<0$}{
                        Set $\hat{r} = r_{l}$; Flag1=1.\\
                        break.
                    }                            
                    Compute ($a_{l+1}$,$r_{l+1}$) by (\ref{l1l2-ar2})-(\ref{l1l2-ar4}). \\
                    \If{$\dfrac{|r_{l+1}-r_{l}|}{|r_{l}|}<\epsilon$}{
                        Set $\hat{r} = r_{l+1}$;\\
                        break.                                             
                    }                     % 计算结束
                } %for
                \eIf{Flag1=1}{ 
                        \eIf{$\sigma_{\dot{\bm{Y}},t+1} > \dfrac{\lambda}{r}$}{
                            $t_1 = t$,}{
                            $t_2 = t$,
                            }
                    }{
                    Compute $\hat{a}$ by (\ref{l1l2-ar1})-(\ref{l1l2-ar2}) with $r_{l} = \hat{r}$.\\
                    Flag2 = $(\sigma_{\dot{\bm{Y}},t} - \lambda/\hat{r})(1-\hat{a}\lambda/\hat{r}^{3})>0$, \\
                    \eIf{Flag2}{
                        \eIf(\tcp*[f]{condition (\ref{2.7})}){$(\sigma_{\dot{\bm{Y}},t+1} - \lambda/\hat{r})(1-      \hat{a}\lambda/\hat{r}^{3}) \leq 0$}{
                            Compute $\bm{\tilde{\Sigma}_{\dot{X}}}$ via (\ref{L1-L2-S}) \tcp*[f]{solution}}{
                            $t_1 = t$.
                            } 
                    }{
                    $t_2 = t$.
                    }
                } % if
            } % while
        } % else
		\KwOut{ $\bm{\tilde{\Sigma}_{\dot{X}}}$}
\end{algorithm}

\section{QNOF for robust matrix completion}
\label{sec:alg}

% \lipsum[40]

\subsection{QNOF for matrix completion}
In this section, we address the matrix completion (MC) problem using the QNOF-based model:
\begin{equation}
\begin{aligned}
\min_{\dot{\bm{X}}} ||\dot{\bm{X}} ||_{\mathrm{QNOF}} \quad \mathrm{s.t.} \quad \mathcal{P}_{\Omega}(\dot{\bm{X}}) = \mathcal{P}_{\Omega}(\dot{\bm{Y}}),
\label{mc-NOF}
\end{aligned}
\end{equation}
%where $\Omega$ is a binary mask matrix of the same size as $\dot{\bm{Y}}$. Zeroes in $\Omega$ indicate unobserved values. $\mathcal{P}_{\Omega}(\dot{\bm{Y}})=\Omega\odot\dot{\bm{Y}}$, where $\odot$ denotes the Hadamard product. This constraint ensures consistency between the estimated $\dot{\bm{X}}$ and the observed $\dot{\bm{Y}}$ at the observed entries.
where $\Omega$ is a binary mask matrix matching the dimensions of $\dot{\bm{Y}}$. The entries of $\Omega$ are zero for unobserved values, and $\mathcal{P}_{\Omega}(\dot{\bm{Y}})=\Omega\odot\dot{\bm{Y}}$, with $\odot$ denoting the Hadamard product. This constraint enforces consistency between the observed entries of $\dot{\bm{Y}}$ and the corresponding entries of the estimated $\dot{\bm{X}}$.

By introducing the auxiliary variable $\dot{\bm{Z}}$, and the Lagrange multiplier $\dot{\bm{\eta}}$,  the augmented Lagrange function for (\ref{mc-NOF}) is expressed as:
\begin{equation}
\begin{aligned}
\mathcal{L}(\dot{\bm{X}}, \dot{\bm{Z}}, \dot{\bm{\eta}}, \beta) &=  \lambda\left(\frac{||\dot{\bm{X}}||_{*}}{||\dot{\bm{X}}||_{F}}\right) + \frac{\beta}{2}||\dot{\bm{Y}}-\dot{\bm{X}} - \dot{\bm{Z}} ||_{F}^{2} + \langle{ \dot{\bm{\eta}},\dot{\bm{Y}}-\dot{\bm{X}} - \dot{\bm{Z}} }\rangle \\
& \mathrm{s.t.} \quad \mathcal{P}_{\Omega}(\dot{\bm{Z}}) = 0. 
\label{en-mc}
\end{aligned}
\end{equation}

The update strategy for (\ref{en-mc}) is defined as: 
\begin{align}\left\{\begin{aligned}
\dot{\bm{Z}}^{(k+1)} &= \arg\min_{\dot{\bm{Z}}} \mathcal{L}(\dot{\bm{X}}^{(k)}, \dot{\bm{Z}}, \dot{\bm{\eta}}^{(k)}, \beta^{(k)}),  \quad \mathrm{s.t.} \quad ||\mathcal{P}_{\Omega}(\dot{\bm{Z}})||_{F}^{2} = 0,  \\
\dot{\bm{X}}^{(k+1)} &= \arg\min_{\dot{\bm{X}}} \mathcal{L}(\dot{\bm{X}}, \dot{\bm{Z}}^{(k+1)}, \dot{\bm{\eta}}^{(k)}, \beta^{(k)}), \\
\dot{\bm{\eta}}^{(k+1)} &= \dot{\bm{\eta}}^{(k)} + \beta^{(k)}(\dot{\bm{Y}}- \dot{\bm{X}}^{(k+1)} - \dot{\bm{Z}}^{(k+1)}), \\
\beta^{(k+1)} &= \mu\beta^{(k)},\quad(\mu>1). \\
\end{aligned}\right.
\label{admm-mc}
\end{align}
The $\dot{\bm{Z}}$ subproblem is given by:
\begin{equation}
\begin{aligned}
\dot{\bm{Z}} = \arg\min_{\dot{\bm{Z}}} \, ||\dot{\bm{Y}} -  \dot{\bm{Z}} + \frac{ \dot{\bm{\eta}}}{\beta} - \dot{\bm{X}}||_{F}^{2} \quad \mathrm{s.t.} \quad ||\mathcal{P}_{\Omega}(\dot{\bm{Z}})||_{F}^{2} = 0.
\end{aligned}
\label{z-mc}
\end{equation}
This quadratic problem can be solved efficiently.

The $\dot{\bm{X}}$ subproblem is formulated as:
\begin{equation}
\begin{aligned}
\dot{\bm{X}} = \arg\min_{\dot{\bm{X}}} \, \lambda\left(\frac{||\dot{\bm{X}}||_{*}}{||\dot{\bm{X}}||_{F}} \right) + \frac{\beta}{2}||(\dot{\bm{Y}} -  \dot{\bm{Z}} + \frac{ \dot{\bm{\eta}}}{\beta}) - \dot{\bm{X}}||_{F}^{2}.
\end{aligned}
\label{x-mc}
\end{equation}
The solution to this subproblem is detailed in Algorithm \ref{alg:1}, and the complete procedure is outlined in Algorithm \ref{alg:mc}.

% \begin{theorem}
% Suppose that the parameter sequence $\{\beta^{(k)}\}$ is unbounded and the sequences $\{\dot{\bm{X}}^{(k)}\}$ and $\{\dot{\bm{Z}}^{(k)}\}$ are generated by Algorithm \ref{ag2}, which satisfies: 
% \begin{align}
% &\lim_{k\rightarrow +\infty} || \dot{\bm{X}}^{(k+1)}-\dot{\bm{X}}^{(k)} ||_{F} = 0, \label{converge21}\\
% &\lim_{k\rightarrow +\infty} || \dot{\bm{Y}} - \dot{\bm{X}}^{(k+1)}-\dot{\bm{Z}}^{(k+)} ||_{F} = 0. \label{converge22}
% \end{align}
% \label{theorem03}
% \end{theorem}

\begin{algorithm}[!t]
	\caption{QNOF-MC Algorithm}
	\label{alg:mc}
    % \LinesNumbered 
		\KwIn{ Observation $\dot{\bm{Y}}$; Initialize $\dot{\bm{X}}^{(0)}=\dot{\bm{Y}}$, $\dot{\bm{Z}}^{(0)}=0$, $\dot{\bm{\eta}}^{(0)}=0$; Parameters  $\lambda$, $\mu$, $\beta$, $K$;\\ }%\\

        \For{$k=0,1,\dots,K-1$}{  % for 更新(a,r)
             Update $\dot{\bm{Z}}^{(k+1)}$ by (\ref{z-mc})   % P
             
             Update $\dot{\bm{X}}^{(k+1)}$ by (\ref{x-mc})   % X
             
             Update $\dot{\bm{\eta}}^{(k+1)} = \dot{\bm{\eta}}^{(k)} + \beta^{(k)}(\dot{\bm{Y}}- \dot{\bm{X}}^{(k+1)} - \dot{\bm{Z}}^{(k+1)})$ 

             Update   $\beta^{(k+1)} = \mu\beta^{(k)},\quad(\mu>1)$ 
             } %for

		\KwOut{ The reconstructed image $\dot{\bm{X}}^{(K)}$ }
\end{algorithm}

\subsection{QNOF for robust principal component analysis}
To address the contamination of observed data by outliers or sparse noise, \cite{candes2011robust} proposed robust principal component analysis (RPCA) using the nuclear norm, with the $L_1$-norm modeling the errors. Building on this approach, we extend QNOF to RPCA, resulting in the QNOF-RPCA model:
\begin{equation}
\begin{aligned}
\min_{\dot{\bm{X}},\dot{\bm{Z}}} \lambda||\dot{\bm{X}} ||_{\mathrm{QNOF}} + \rho||\dot{\bm{Z}}||_{1} \quad \mathrm{s.t.} \quad \dot{\bm{Y}} = \dot{\bm{X}} + \dot{\bm{Z}}.
\label{rpca-NOF}
\end{aligned}
\end{equation}
where $\dot{\bm{Z}}$ is sparse error. According to ADMM, its augmented Lagrange function is
\begin{equation}
\begin{aligned}
\mathcal{L}(\dot{\bm{X}}, \dot{\bm{Z}}, \dot{\bm{\eta}}, \beta) &=  \lambda||\dot{\bm{X}} ||_{\mathrm{QNOF}} + \rho||\dot{\bm{Z}}||_{1} + \frac{\beta}{2}||\dot{\bm{Y}}-\dot{\bm{X}} - \dot{\bm{Z}} ||_{F}^{2}
 + \langle{ \dot{\bm{\eta}},\dot{\bm{Y}}-\dot{\bm{X}} - \dot{\bm{Z}} }\rangle.
\label{en-rpca}
\end{aligned}
\end{equation}

The update strategy for (\ref{en-rpca}) is as follows:
\begin{align}\left\{\begin{aligned}
\dot{\bm{Z}}^{(k+1)} &= \arg\min_{\dot{\bm{Z}}} \mathcal{L}(\dot{\bm{X}}^{(k)}, \dot{\bm{Z}}, \dot{\bm{\eta}}^{(k)}, \beta^{(k)}),   \\
\dot{\bm{X}}^{(k+1)} &= \arg\min_{\dot{\bm{X}}} \mathcal{L}(\dot{\bm{X}}, \dot{\bm{Z}}^{(k+1)}, \dot{\bm{\eta}}^{(k)}, \beta^{(k)}), \\
\dot{\bm{\eta}}^{(k+1)} &= \dot{\bm{\eta}}^{(k)} + \beta^{(k)}(\dot{\bm{Y}}- \dot{\bm{X}}^{(k+1)} - \dot{\bm{Z}}^{(k+1)}), \\
\beta^{(k+1)} &= \mu\beta^{(k)},\quad(\mu>1). \\
\end{aligned}\right.
\label{admm-rpca}
\end{align}
The $\dot{\bm{Z}}$ subproblem is
\begin{equation}
\begin{aligned}
\dot{\bm{Z}} = \arg\min_{\dot{\bm{Z}}} \, \rho||\dot{\bm{Z}}||_1 + \frac{\beta}{2} ||(\dot{\bm{Y}} -  \dot{\bm{X}} + \frac{ \dot{\bm{\eta}}}{\beta}) - \dot{\bm{Z}}||_{F}^{2}.
\end{aligned}
\label{z-rpca}
\end{equation}
According to the soft-threshold shrinkage, the solution for $\dot{\bm{Z}}$  is
\begin{equation}
\begin{aligned}
\dot{\bm{Z}} = \frac{\dot{\bm{Y}} -  \dot{\bm{X}} + \frac{ \dot{\bm{\eta}}}{\beta}}{\left|\dot{\bm{Y}} -  \dot{\bm{X}} + \frac{ \dot{\bm{\eta}}}{\beta}\right| + \epsilon }\max\left(\left|\dot{\bm{Y}} -  \dot{\bm{X}} + \frac{ \dot{\bm{\eta}}}{\beta}\right| - \frac{\rho}{\beta}, 0\right),
\end{aligned}
\label{z-s}
\end{equation}
where $|\cdot|$ is  the modulus of the quaternion and $\epsilon$ is a small number. The $\dot{\bm{X}}$ subproblem is consistent with \eqref{x-mc}. The whole algorithmic procedure is described in Algorithm \ref{alg:rpca}.

\begin{algorithm}[!t]
	\caption{QNOF-RPCA Algorithm}
	\label{alg:rpca}
    % \LinesNumbered 
		\KwIn{ Observation $\dot{\bm{Y}}$; Initialize $\dot{\bm{X}}^{(0)}=\dot{\bm{Y}}$, $\dot{\bm{Z}}^{(0)}=0$, $\dot{\bm{\eta}}^{(0)}=0$; Parameters $\rho$, $\lambda$, $\mu$, $\beta$, $K$;\\ }%\\

        \For{$k=0,1,\dots,K-1$}{  % for 更新(a,r)
             
             Update $\dot{\bm{Z}}^{(k+1)}$ by (\ref{z-s})   % Z
            
             Update $\dot{\bm{X}}^{(k+1)}$ by (\ref{x-mc})   % X

             Update $\dot{\bm{\eta}}^{(k+1)} = \dot{\bm{\eta}}^{(k)} + \beta^{(k)}(\dot{\bm{Y}}- \dot{\bm{X}}^{(k+1)} - \dot{\bm{Z}}^{(k+1)})$ 
         
             Update   $\beta^{(k+1)} = \mu\beta^{(k)},\quad(\mu>1)$ 
             } %for

		\KwOut{ The reconstructed image $\dot{\bm{X}}^{(K)}$ }
\end{algorithm}

\subsection{QNOF for robust matrix completion}
Combining \eqref{mc-NOF} and \eqref{rpca-NOF}, we formulate the QNOF-based robust matrix completion (RMC) problem as: 
\begin{equation}
\begin{aligned}
\min_{\dot{\bm{X}}} \lambda||\dot{\bm{X}} ||_{\mathrm{QNOF}} + \rho||\dot{\bm{Z}}||_{1} \quad \mathrm{s.t.} \quad \mathcal{P}_{\Omega}(\dot{\bm{X}} + \dot{\bm{Z}}) = \mathcal{P}_{\Omega}(\dot{\bm{Y}}),
\label{rmc-NOF}
\end{aligned}
\end{equation}
Introducing the auxiliary variables $\dot{\bm{P}}$ and $\dot{\bm{Q}}$, along with Lagrange multipliers $\dot{\bm{\eta}}$ and $\dot{\bm{\xi}}$, the augmented Lagrange function for \eqref{rmc-NOF} is given by:
\begin{equation}
\begin{aligned}
\mathcal{L}(\dot{\bm{X}}, \dot{\bm{Z}}, \dot{\bm{P}}, \dot{\bm{Q}}, \dot{\bm{\eta}}, \dot{\bm{\xi}}, &\beta_1, \beta_2) =  \lambda\left(\frac{||\dot{\bm{X}}||_{*}}{||\dot{\bm{X}}||_{F}}\right) + \rho||\dot{\bm{Z}}||_{1} + \frac{\beta_1}{2}||\dot{\bm{X}} - \dot{\bm{P}} ||^{2}_{F} + \langle{ \dot{\bm{\eta}},\dot{\bm{X}}-\dot{\bm{P}} }\rangle \\
&+  \frac{\beta_2}{2}||\dot{\bm{Z}} - \dot{\bm{Q}} ||^{2}_{F} + \langle{ \dot{\bm{\xi}},\dot{\bm{Z}}-\dot{\bm{Q}} }\rangle ~~ \mathrm{s.t.} ~~ \mathcal{P}_{\Omega}(\dot{\bm{P}} + \dot{\bm{Q}}) = \mathcal{P}_{\Omega}(\dot{\bm{Y}}).
\label{E-rmc}
\end{aligned}
\end{equation}

The update strategy for solving (\ref{E-rmc}) is defined as:
\begin{align}\left\{\begin{aligned}
\dot{\bm{P}}^{(k+1)} &= \arg\min_{\dot{\bm{P}}} \mathcal{L}(\dot{\bm{X}}^{(k)}, \dot{\bm{Z}}^{(k)}, \dot{\bm{P}}, \dot{\bm{Q}}^{(k)}, \dot{\bm{\eta}}^{(k)}, \dot{\bm{\xi}}^{(k)}, \beta_1^{(k)}, \beta_2^{(k)}),   ~~ \mathrm{s.t.} ~~ \mathcal{P}_{\Omega}(\dot{\bm{P}} + \dot{\bm{Q}}) = \mathcal{P}_{\Omega}(\dot{\bm{Y}}), \\
\dot{\bm{X}}^{(k+1)} &= \arg\min_{\dot{\bm{X}}} \mathcal{L}(\dot{\bm{X}}, \dot{\bm{Z}}^{(k)}, \dot{\bm{P}}^{(k+1)}, \dot{\bm{Q}}^{(k)}, \dot{\bm{\eta}}^{(k)}, \dot{\bm{\xi}}^{(k)}, \beta_1^{(k)}, \beta_2^{(k)}),   \\
\dot{\bm{Z}}^{(k+1)} &= \arg\min_{\dot{\bm{Z}}} \mathcal{L}(\dot{\bm{X}}^{(k+1)}, \dot{\bm{Z}}, \dot{\bm{P}}^{(k+1)}, \dot{\bm{Q}}^{(k)}, \dot{\bm{\eta}}^{(k)}, \dot{\bm{\xi}}^{(k)}, \beta_1^{(k)}, \beta_2^{(k)}),   \\
\dot{\bm{Q}}^{(k+1)} &= \arg\min_{\dot{\bm{Q}}} \mathcal{L}(\dot{\bm{X}}^{(k+1)}, \dot{\bm{Z}}^{(k+1)}, \dot{\bm{P}}^{(k+1)}, \dot{\bm{Q}}, \dot{\bm{\eta}}^{(k)}, \dot{\bm{\xi}}^{(k)}, \beta_1^{(k)}, \beta_2^{(k)}), ~~ \mathrm{s.t.} ~~ \mathcal{P}_{\Omega}(\dot{\bm{P}} + \dot{\bm{Q}}) = \mathcal{P}_{\Omega}(\dot{\bm{Y}}), \\
\dot{\bm{\eta}}^{(k+1)} &= \dot{\bm{\eta}}^{(k)} + \beta_{1}^{(k)}(\dot{\bm{X}}^{(k+1)}- \dot{\bm{P}}^{(k+1)}), \\
\dot{\bm{\xi}}^{(k+1)} &= \dot{\bm{\xi}}^{(k)} + \beta_{2}^{(k)}(\dot{\bm{Z}}^{(k+1)} - \dot{\bm{Q}}^{(k+1)}), \\
\beta_{1}^{(k+1)} &= \mu\beta_{1}^{(k)}, \beta_{2}^{(k+1)} = \mu\beta_{2}^{(k)},\quad(\mu>1). \\
\end{aligned}\right.
\label{admm-rmc}
\end{align}
The $\dot{\bm{X}}$ subproblem is solved as:
\begin{equation}
\begin{aligned}
\dot{\bm{X}} = \arg\min_{\dot{\bm{X}}} \, \lambda\left(\frac{||\dot{\bm{X}}||_{*}}{||\dot{\bm{X}}||_{F}}\right) +  \frac{\beta_1}{2}||\dot{\bm{X}} - \dot{\bm{P}} + \frac{\dot{\bm{\eta}}}{\beta_1}||^{2}_{F}. 
\label{X-rmc}
\end{aligned}
\end{equation}
Theorem \ref{theorem01} and Algorithm \ref{alg:1} provide the solution to this subproblem.

The $\dot{\bm{Z}}$ subproblem is formulated as:
\begin{equation}
\begin{aligned}
\dot{\bm{Z}} = \arg\min_{\dot{\bm{Z}}} \,  \rho||\dot{\bm{Z}}||_{1} +  \frac{\beta_2}{2}||\dot{\bm{Z}} - \dot{\bm{Q}} + \frac{\dot{\bm{\xi}}}{\beta_2} ||^{2}_{F}. 
\label{Z-rmc}
\end{aligned}
\end{equation}
This formulation aligns with \eqref{z-s}.

The $\dot{\bm{P}}$ and $\dot{\bm{Q}}$ subproblems are defined as:
\begin{equation}
\begin{aligned}
\dot{\bm{P}} = \arg\min_{\dot{\bm{P}}} \,  \frac{\beta_1}{2}||\dot{\bm{X}} - \dot{\bm{P}} + \frac{\dot{\bm{\eta}}}{\beta_1}||^{2}_{F} ~~ \mathrm{s.t.} ~~ \mathcal{P}_{\Omega}(\dot{\bm{P}} + \dot{\bm{Q}}) = \mathcal{P}_{\Omega}(\dot{\bm{Y}}), \\
\dot{\bm{Q}} = \arg\min_{\dot{\bm{Q}}} \, \frac{\beta_2}{2}||\dot{\bm{Z}} - \dot{\bm{Q}} + \frac{\dot{\bm{\xi}}}{\beta_2} ||^{2}_{F} ~~ \mathrm{s.t.} ~~ \mathcal{P}_{\Omega}(\dot{\bm{P}} + \dot{\bm{Q}}) = \mathcal{P}_{\Omega}(\dot{\bm{Y}}). 
\label{PQ-rmc}
\end{aligned}
\end{equation}
The closed-form solutions for these quadratic problems are:
\begin{align}
&\dot{\bm{P}} =\left\{\begin{aligned}
&\dot{\bm{X}}(\bar{\Omega}) + \frac{\dot{\bm{\eta}}(\bar{\Omega})}{\beta_1}, \, &\dot{P_{ij}} \in \bar{\Omega};\\
&\frac{\dot{\bm{X}}(\Omega) + \frac{\dot{\bm{\eta}}(\Omega)}{\beta_1} + \dot{\bm{Y}}(\Omega) - \dot{\bm{Z}}(\Omega) - \frac{\dot{\bm{\xi}}(\Omega)}{\beta_2}}{2}, \, &\dot{P_{ij}} \in \Omega;\\
\end{aligned}\right. \label{PS-rmc} \\
&\dot{\bm{Q}} =\left\{\begin{aligned}
& \dot{\bm{Z}}(\bar{\Omega}) + \frac{\dot{\bm{\xi}}(\bar{\Omega})}{\beta_2}, \, \dot{Q_{ij}} \in \bar{\Omega};\\
& \dot{\bm{Y}}(\Omega) - \dot{\bm{P}}(\Omega), \, \dot{Q_{ij}} \in \Omega.
\end{aligned}\right.
\label{QS-rmc}
\end{align}

\begin{algorithm}[!t]
	\caption{QNOF-RMC Algorithm}
	\label{alg:rmc}
    % \LinesNumbered 
		\KwIn{ Observation $\dot{\bm{Y}}$; Initialize $\dot{\bm{X}}^{(0)}=\dot{\bm{Y}}$, $\dot{\bm{Z}}^{(0)}=0$, $\dot{\bm{\eta}}^{(0)}=0$; Parameters $\rho$, $\lambda$, $\mu$, $\beta_1$, $\beta_2$, $K$;\\ }%\\

        \For{$k=0,1,\dots,K-1$}{  % for 更新(a,r)
             Update $\dot{\bm{P}}^{(k+1)}$ by \eqref{PS-rmc}   % P
             
             Update $\dot{\bm{X}}^{(k+1)}$ by \eqref{X-rmc}   % X
             
             Update $\dot{\bm{Z}}^{(k+1)}$ by \eqref{Z-rmc}   % Z
             
             Update $\dot{\bm{Q}}^{(k+1)}$ by \eqref{QS-rmc}   % Q

             Update $\dot{\bm{\eta}}^{(k+1)} = \dot{\bm{\eta}}^{(k)} + \beta_{1}^{(k)}(\dot{\bm{X}}^{(k+1)}- \dot{\bm{P}}^{(k+1)})$ 
             
             Update $\dot{\bm{\xi}}^{(k+1)} = \dot{\bm{\xi}}^{(k)} + \beta_{2}^{(k)}(\dot{\bm{Z}}^{(k+1)} - \dot{\bm{Q}}^{(k+1)})$ 
             
             Update   $\beta_{1}^{(k+1)} = \mu\beta_{1}^{(k)}, \beta_{2}^{(k+1)} = \mu\beta_{2}^{(k)},\quad(\mu>1)$ 
             } %for

		\KwOut{ The reconstructed image $\dot{\bm{X}}^{(K)}$ }
\end{algorithm}

\subsection{Convergence analysis}

This section analyzes the convergence of the proposed ADMM schemes \eqref{admm-mc}, \eqref{admm-rmc}, and \eqref{admm-rpca}. Since \eqref{admm-mc} and \eqref{admm-rpca} are special cases of \eqref{admm-rmc}, it suffices to demonstrate the convergence of \eqref{admm-rmc}. In Theorem \ref{theorem04}, we prove that the sequences generated by \eqref{admm-rmc} are weakly convergent. Specifically, these sequences are Cauchy sequences when the sequence $\{\dot{\bm{X}}^{(k)}\}$ is bounded.

\begin{theorem}
Suppose that the parameter sequence $\{\beta_1^{(k)}\}$, $\{\beta_2^{(k)}\}$ are unbounded and the sequences $\{\dot{\bm{X}}^{(k)}\}$, $\{\dot{\bm{Z}}^{(k)}\}$, $\{\dot{\bm{P}}^{(k)}\}$, and $\{\dot{\bm{Q}}^{(k)}\}$ are generated by Algorithm \ref{alg:rmc}. Additionally, assume $\{\dot{\bm{X}}^{(k)}\}$ is bounded. To prevent the matrices from degenerating to $\bm{0}$, let a small constant $\theta$ exist such that the first singular value $\sigma_{1}^{(k)}$ of $\dot{\bm{P}}^{(k+1)} - {\beta_{1}^{(k)}}^{-1}\dot{\bm{\eta}}^{(k)}$ satisfies $\sigma_{1}^{(k)} > \theta$. Under these conditions, the following convergence properties hold:

\begin{align}
&\lim_{k\rightarrow +\infty} || \dot{\bm{X}}^{(k+1)}-\dot{\bm{X}}^{(k)} ||_{F} = 0, \label{converge31}\\
&\lim_{k\rightarrow +\infty} || \dot{\bm{Z}}^{(k+1)}-\dot{\bm{Z}}^{(k)} ||_{F} = 0. \label{converge32}\\
&\lim_{k\rightarrow +\infty} || \dot{\bm{P}}^{(k+1)}-\dot{\bm{P}}^{(k)} ||_{F} = 0. \label{converge33}\\
&\lim_{k\rightarrow +\infty} || \dot{\bm{Q}}^{(k+1)}-\dot{\bm{Q}}^{(k)} ||_{F} = 0. \label{converge34}\\
&\lim_{k\rightarrow +\infty} ||\dot{\bm{X}}^{(k+1)}-\dot{\bm{P}}^{(k+1)} ||_{F} = 0. \label{converge35}\\
&\lim_{k\rightarrow +\infty} ||\dot{\bm{Z}}^{(k+1)}-\dot{\bm{Q}}^{(k+1)} ||_{F} = 0. \label{converge36}
\end{align}
\label{theorem04}
\end{theorem}

See Appendix \ref{appendix.a} for the proof. Notably, the boundedness of $\{\dot{\bm{X}}^{(k)}\}$ is essential to ensure convergence. The term $||\dot{\bm{X}} ||_{\mathrm{QNOF}}$ differs from $||\dot{\bm{X}}||_{*}$ and $||\dot{\bm{X}}||_{*} - \alpha||\dot{\bm{X}}||_{F}$. Fortunately, this assumption is generally satisfied in practical applications. For instance, color images represented as quaternion matrices are inherently bounded.

\section{Experimental results}
\label{sec:experiments}

In this section, we perform a series of numerical tests. All experiments are performed on a PC with Intel (R) Core (TM) i7-13700KF 3.40 GHz and 64G RAM and MATLAB (R2022b).

\subsection{Simulation data}

We begin by testing the feasibility of QNOF using small-scale simulation data. We generate random quaternion matrices $\dot{\bm{X}}_0$ of size $n \times n$, where $n$ is set to 50 and 100. For $n=50$, we vary the rank of $\dot{\bm{X}}_0$ to 2, 4, 6, 8, and 10; for $n=100$, we set the ranks to 2, 6, 10, 15, and 20. These matrices are corrupted with 5\% random noise, leaving 95\% of the entries observed. The recovery results of QNOF are presented in Table \ref{tab:sim}. The findings demonstrate that the $\mathrm{rank}(\dot{\bm{X}})$ of the recovered matrix aligns with $\mathrm{rank}(\dot{\bm{X}}_0)$, and the relative error $||\dot{\bm{X}}-\dot{\bm{X}}_0||_{F}/||\dot{\bm{X}}_0||_{F}$ remains small.

\begin{table}[t]
\footnotesize
\caption{ Performance of QNOF at 5\% noise damage and 95\% observation. }
\label{tab:sim}
\begin{center}
\renewcommand\arraystretch{0.9}
  \begin{tabular}{lccccc} 
  \hline
   \multicolumn{6}{c}{$n=50$} \\ \hline
    rank($\dot{\bm{X}}_0$) & 2 & 4 & 6 & 8 & 10 \\
    rank($\dot{\bm{X}}$)   & 2 & 4 & 6 & 8 & 10  \\ 
    $||\dot{\bm{X}}-\dot{\bm{X}}_0||_{F}/||\dot{\bm{X}}_0||_{F}$ & 2.3634e-08 & 3.1679e-08 & 3.4732e-08 & 4.0861e-08 & 4.6718e-08\\ \hline

    \multicolumn{6}{c}{$n=100$} \\ \hline
    rank($\dot{\bm{X}}_0$) & 2 & 6 & 10 & 15 & 20  \\
    rank($\dot{\bm{X}}$)   & 2 & 6 & 10 & 15 & 20  \\ 
    $||\dot{\bm{X}}-\dot{\bm{X}}_0||_{F}/||\dot{\bm{X}}_0||_{F}$ & 2.3283e-08 & 2.1213e-08 & 3.1720e-08 & 5.6182e-08 & 5.7727e-08 \\ \hline

  \end{tabular}
\end{center}
\end{table}

\begin{figure}[t]
	\centering
	\addtolength{\tabcolsep}{-5.5pt}
    \renewcommand\arraystretch{0.4}
	{\fontsize{8pt}{\baselineskip}\selectfont 
	\begin{tabular}{ccc}
        \includegraphics[width=0.3\textwidth]{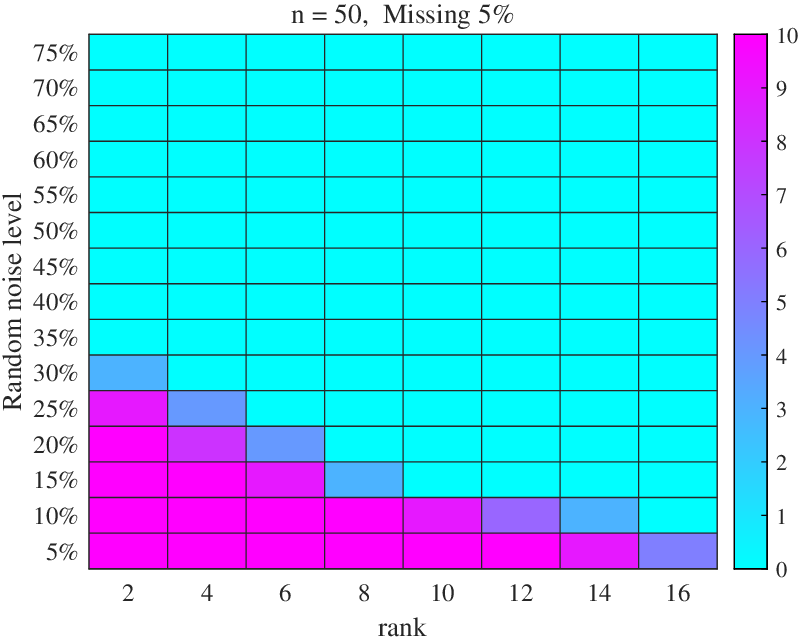} &
        \includegraphics[width=0.3\textwidth]{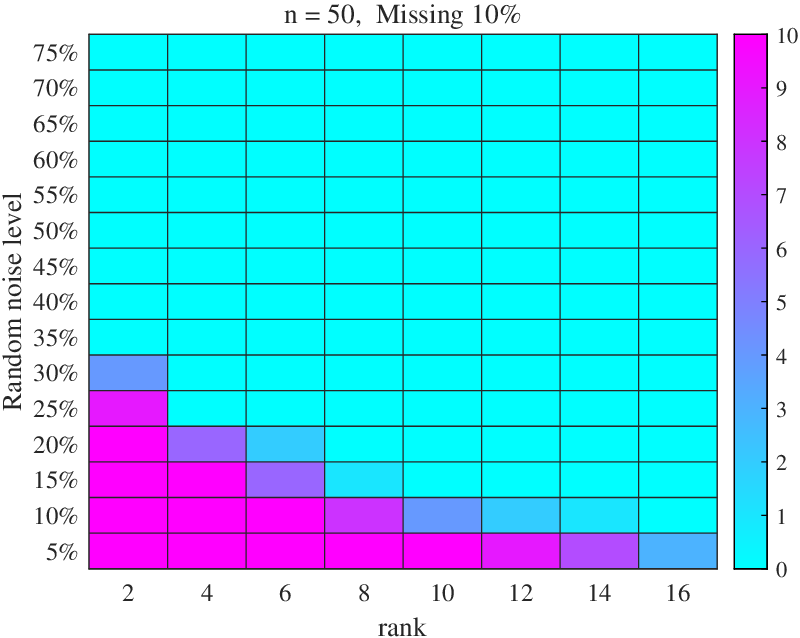} &
        \includegraphics[width=0.3\textwidth]{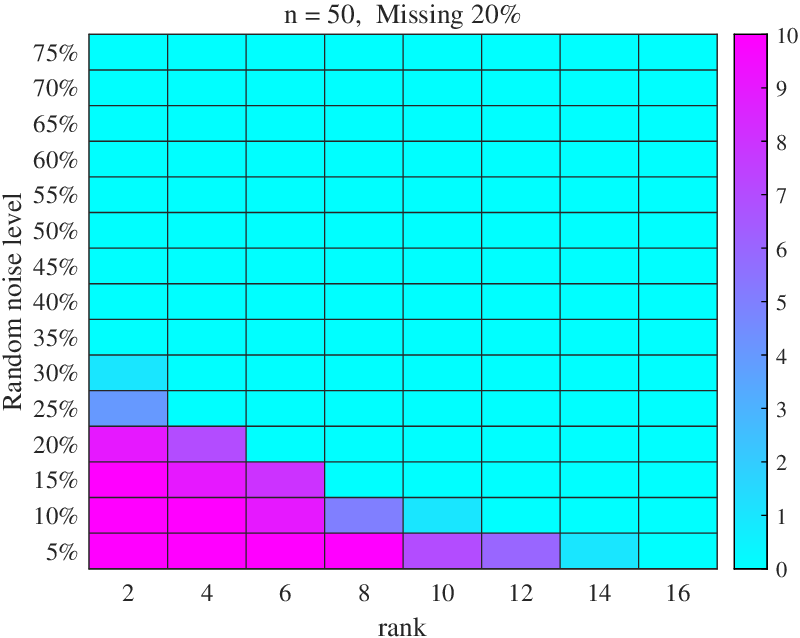} \\
        \includegraphics[width=0.3\textwidth]{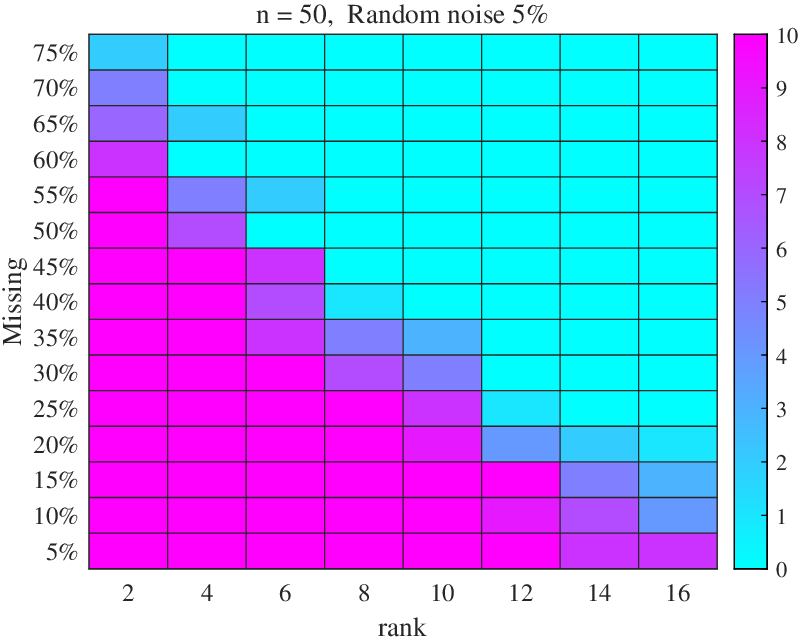} &
        \includegraphics[width=0.3\textwidth]{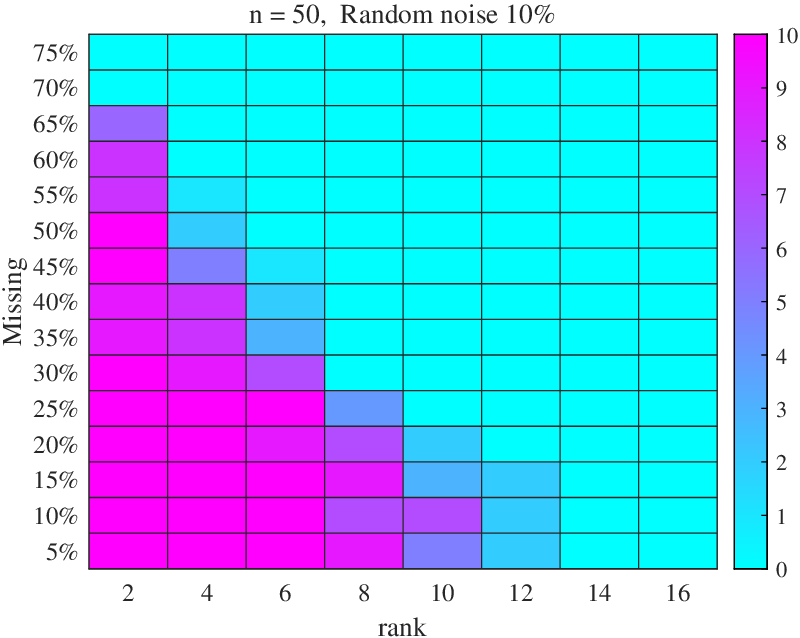} &
        \includegraphics[width=0.3\textwidth]{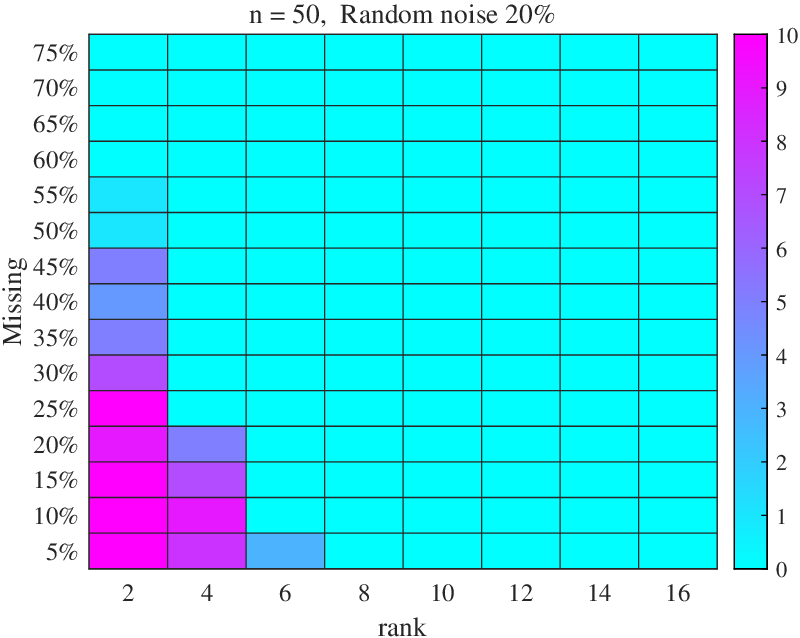} \\

	\end{tabular}
	}
	\caption{Exact recovery for different ranks ($x$-axis) and noise/missing ($y$-axis) degradation is analyzed with a fixed $n=50$. The first row shows the effect of noise and rank on exact recovery with a fixed proportion of missing entries. The second row illustrates the effect of missing entries and rank on exact recovery with a fixed noise level.}
	\label{fig:sim}
\end{figure}
\begin{figure}[t]
	\centering
	\addtolength{\tabcolsep}{-5.5pt}
    \renewcommand\arraystretch{0.4}
	{\fontsize{8pt}{\baselineskip}\selectfont 
	\begin{tabular}{cccccc}
        \includegraphics[width=0.13\textwidth]{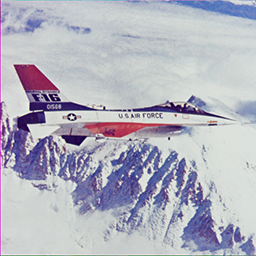} &
        \includegraphics[width=0.13\textwidth]{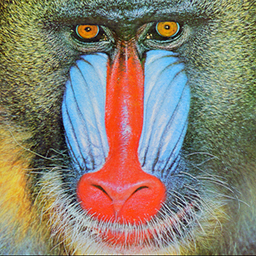} &
        \includegraphics[width=0.13\textwidth]{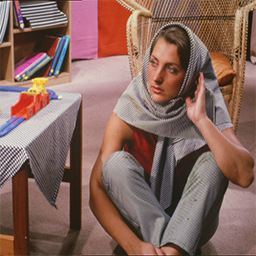} &
        \includegraphics[width=0.13\textwidth]{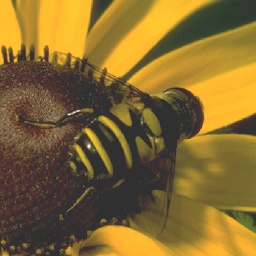} &
        \includegraphics[width=0.13\textwidth]{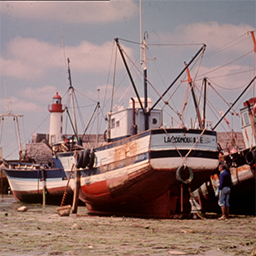} &
        \includegraphics[width=0.13\textwidth]{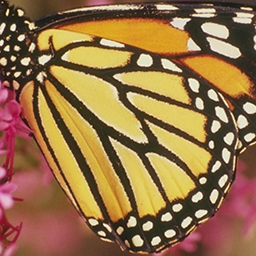} \\
        \includegraphics[width=0.13\textwidth]{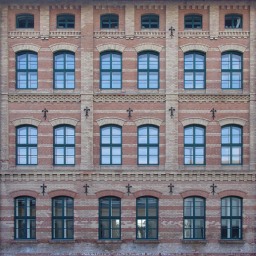} &
        \includegraphics[width=0.13\textwidth]{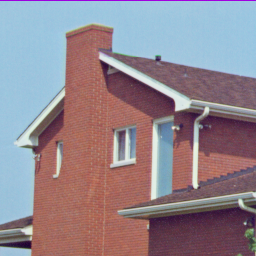} &
        \includegraphics[width=0.13\textwidth]{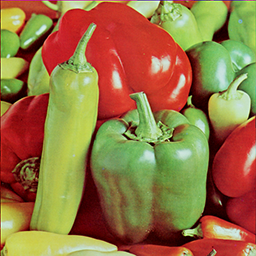} &
        \includegraphics[width=0.13\textwidth]{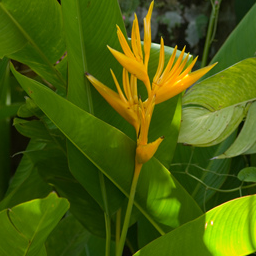} & \includegraphics[width=0.13\textwidth]{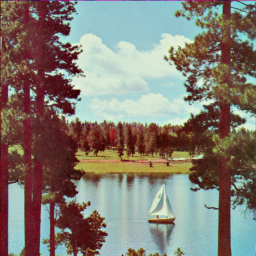} &
        \includegraphics[width=0.13\textwidth]{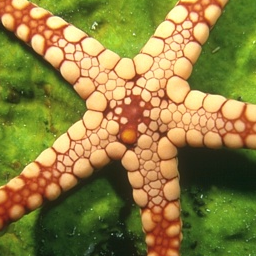} \\

	\end{tabular}
	}
	\caption{ CSet12. }
	\label{CSet12}
\end{figure}

To further evaluate the accuracy of QNOF in matrix recovery, we investigate its relationship with the matrix rank, the proportion of missing entries, and the degree of noise corruption. The experiments are organized into two groups. In the first group, the proportion of missing entries is fixed at 5\%, 10\%, and 20\%, while the ($\mathrm{rank}(\dot{\bm{X}}_0)$, noise level) pairs are varied. Here, $\mathrm{rank}(\dot{\bm{X}}_0)$ increases by 2 within the range $[2,16]$, and the noise level is incremented by 5\% within the range $[5\%, 75\%]$. In the second group, the noise levels are fixed at 5\%, 10\%, and 20\%, while the ($\mathrm{rank}(\dot{\bm{X}}_0)$, missing entries) pairs are varied in the same manner as in the first group. For each experiment, we conduct 10 trials, and the matrix is considered accurately recovered if $||\dot{\bm{X}}-\dot{\bm{X}}_0||_{F}/||\dot{\bm{X}}_0||_{F} \leq 1.0e-8$. The exact recovery rates for each experiment are shown in Fig. \ref{fig:sim}. The results indicate that a smaller proportion of missing entries or noise corruption increases the likelihood of exact recovery. Furthermore, the rank of the matrix also influences recovery success; lower ranks lead to higher recovery rates.

\subsection{Real image data}

\begin{table}[t]
\footnotesize
\caption{ State-of-the-art algorithms compare PSNR/SSIM on MC. The best results are shown in \textbf{bold}. }
\label{tab:mc}
\begin{center}
\resizebox{\textwidth}{!}{
  \begin{tabular}{ccccccccc} 
  \hline
    Miss & QNNM & QWNNM & QSNM$_{1/2}$ & QSNM$_{2/3}$ & QSNM$_{0.975}$ & QWSNM & QNMF & QNOF \\ \hline

    50\%  & 26.12/0.7733 & 24.63/0.6577 & 25.87/0.7156 & 26.30/0.7440 & 26.22/0.7739 & 25.47/0.6993 & 26.27/0.7594 & \textbf{26.34}/\textbf{0.7819} \\ 
    60\%  & 24.18/0.6890 & 22.80/0.5680 & 23.94/0.6255 & 24.34/0.6548 & 24.27/0.6895 & 23.62/0.6124 & 24.36/0.6717 & \textbf{24.43}/\textbf{0.7004} \\ 
    75\%  & 21.24/0.5287 & 19.60/0.3942 & 20.88/0.4568 & 21.32/0.4859 & 21.33/0.5275 & 19.14/0.3827 & 21.52/0.5083 & \textbf{21.59}/\textbf{0.5442} \\ 
    80\%  & 20.08/0.4616 & 17.04/0.2769 & 19.61/0.3879 & 20.10/0.4182 & 20.17/0.4594 & 14.42/0.2097 & 20.37/0.4415 & \textbf{20.46}/\textbf{0.4792} \\ \hline

  \end{tabular}
}
\end{center}
\end{table}

\begin{table}[!t]
\footnotesize
\caption{State-of-the-art algorithms compare PSNR/SSIM on RPCA. The best results are shown in \textbf{bold}.}
\label{tab:rpca}
\begin{center}
\resizebox{\textwidth}{!}{
  \begin{tabular}{ccccccccc} 
  \hline

   Random  & QNNM & QWNNM & QSNM$_{1/2}$ & QSNM$_{2/3}$ & QSNM$_{0.975}$ & QWSNM & QNMF & QNOF \\ \hline
3\%  & 25.78/0.8627 & 24.40/0.7582 & 28.22/0.8756 & 28.82/0.8972 & 29.51/0.9432 & 24.32/0.7513 & \textbf{33.21}/0.9462 & 32.00/\textbf{0.9662} \\ 
5\%  & 25.69/0.8586 & 24.39/0.7558 & 28.11/0.8712 & 28.66/0.8926 & 29.32/0.9375 & 24.29/0.7489 & \textbf{32.07}/0.9362 & 31.58/\textbf{0.9563} \\ 
7\%  & 25.61/0.8544 & 24.33/0.7534 & 28.01/0.8670 & 28.50/0.8877 & 29.11/0.9313 & 24.25/0.7468 & 30.86/0.9164  & \textbf{31.12}/\textbf{0.9441} \\ 
10\%  & 25.47/0.8479 & 24.28/0.7498 & 27.73/0.8576 & 28.25/0.8793 & 28.79/0.9207 & 24.18/0.7438 & 29.34/0.8895 & \textbf{30.38}/\textbf{0.9223} \\ 
20\%  & 24.94/0.8216 & 23.97/0.7338 & 26.78/0.8201 & 27.27/0.8421 & 27.57/0.8691 & 23.91/0.7294 & 25.85/0.7704 & \textbf{27.71}/\textbf{0.8547} \\ \hline
% 30\%  & 24.32/0.7895 & 23.55/0.7109 & 23.30/0.6946 & 25.58/0.7714 & 25.93/0.7854 & 22.87/0.6584 & 25.87/0.7891 \\ \hline

  \end{tabular}
}
\end{center}
\end{table}

We evaluate the proposed QNOF on real images, as shown in Fig. \ref{CSet12}, and compare it with several state-of-the-art quaternion low-rank methods: quaternion nuclear norm (QNNM) \cite{jia2019robust}, quaternion weighted nuclear norm (QWNNM) \cite{yu2019quaternion,huang2022quaternion}, quaternion Schatten $p$-norm (QSNM), quaternion weighted Schatten $p$-norm (QWSNM) \cite{zhang2024quaternion}, and quaternion nuclear norm minus Frobenius norm (QNMF) \cite{guo2025quaternion}. Our evaluation focuses on two tasks: color image inpainting and random impulse noise removal. We assess performance using the peak signal-to-noise ratio (PSNR) and structural similarity (SSIM), defined as:
\begin{equation}
\begin{aligned}
\mathrm{PSNR} &= 10 \log_{10}\frac{{\mathrm{Peak}}^{2}}{\mathrm{MSE}}, \quad \mathrm{MSE}  
= \frac{1}{mn}\sum_{i=1}^{mn} \left( u(i) - \hat{u}(i) \right)^{2}; \\
\mathrm{SSIM} &= \frac{(2{\mu_u}{\mu_{\hat{u}}}+c_1)(2{\sigma_{u\hat{u}}}+c_2)}
{({\mu^2_u}+{\mu^2_{\hat{u}}}+c_1) ({\sigma^2_u}+{\sigma^2_{\hat{u}}}+c_2)},
\end{aligned}
\end{equation}
where Peak represents the image peak value. The variables $u$ and $\hat{u}$ denote the ground truth and restored images, respectively. $\mu_u$ and $\sigma_u$ are the mean and standard deviation of $u$, while $\mu_{\hat{u}}$ and $\sigma_{\hat{u}}$ refer to the statistics of $\hat{u}$. Additionally, $\sigma_{u\hat{u}}$ is the covariance between $u$ and $\hat{u}$, and $c_1$ and $c_2$ are small constants.

For the MC simulations of color image inpainting, we test four missing rates: 50\%, 60\%, 75\%, and 80\%. The numerical results are summarized in Table \ref{tab:mc}, which demonstrates that QNOF achieves the best performance. Notably, the improvement in SSIM is more pronounced, indicating that QNOF excels at preserving image structure and detail. While the recent QNMF method also yields good PSNR results, it is significantly inferior to QNOF in terms of SSIM.

For RPCA, we test five scenarios with random noise rates of 3\%, 5\%, 7\%, 10\%, and 20\% for color image random impulse noise removal. The results are shown in Table \ref{tab:rpca}. Similar to the MC simulations, QNOF outperforms other methods in terms of SSIM. At low noise levels, QNMF performs better than QNOF, but at higher noise levels ($> 5\%$), QNOF demonstrates the best performance.

For RMC, we validate the method through a joint task of color image inpainting and random impulse noise removal. We consider six combinations of missing and noise rates: (50\%, 3\%), (50\%, 5\%), (70\%, 3\%), (70\%, 5\%), (80\%, 3\%), and (80\%, 5\%). The results, presented in Table \ref{tab:rmc}, show that QNOF remains robust and achieves optimal performance across all scenarios.

\begin{table}[!t]
\footnotesize
\caption{State-of-the-art algorithms compare PSNR/SSIM on RMC. The best results are shown in \textbf{bold}.}
\label{tab:rmc}
\begin{center}
\resizebox{\textwidth}{!}{
  \begin{tabular}{ccccccccc} 
  \hline
    (Miss, Random)  & QNNM & QWNNM & QSNM$_{1/2}$ & QSNM$_{2/3}$ & QSNM$_{0.975}$ & QWSNM & QNMF & QNOF \\ \hline

        (50\%, 3\%)  & 24.63/0.7629 & 21.09/0.6007 & 25.03/0.7406 & 24.99/0.7505 & 24.90/0.7650 & 24.09/0.7124 & 24.36/0.7476 & \textbf{25.29}/\textbf{0.7520} \\ 
        (50\%, 5\%)  & 24.47/0.7479 & 21.02/0.5963 & 24.81/0.7245 & 24.83/0.7373 & 24.70/0.7484 & 23.96/0.7010 & 24.24/0.7359 & \textbf{24.80}/\textbf{0.7423} \\ 
        (70\%, 3\%)  & 21.77/0.5917 & 19.64/0.5092 & 21.90/0.5485 & 20.70/0.5634 & 21.93/0.5903 & 20.72/0.4906 & 21.34/0.5866 & \textbf{22.03}/\textbf{0.5919} \\ 
        (70\%, 5\%)  & 21.65/0.5761 & 19.61/0.5033 & 21.47/0.5127 & 20.67/0.5582 & 21.78/0.5727 & 19.73/0.4288 & 21.25/0.5762 & \textbf{21.81}/\textbf{0.5802} \\ 
        (80\%, 3\%)  & 19.88/0.4652 & 18.25/0.4258 & 17.04/0.2863 & 19.79/0.4904 & 19.96/0.4591 & 13.91/0.1856 & 18.72/0.4157 & \textbf{20.04}/\textbf{0.4916} \\ 
        (80\%, 5\%)  & 19.72/0.4501 & 18.17/0.4195 & 16.32/0.2502 & 19.69/0.4838 & 19.78/0.4423 & 13.23/0.1530 & 18.58/0.4027 & \textbf{19.92}/\textbf{0.4791} \\ \hline

  \end{tabular}
}
\end{center}
\end{table}

\begin{figure}[!t]
	\centering
	\addtolength{\tabcolsep}{-5pt}
    \renewcommand\arraystretch{0.4}
	{\fontsize{8pt}{\baselineskip}\selectfont 
	\begin{tabular}{ccccc}
        \includegraphics[width=0.17\textwidth]{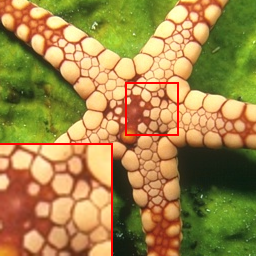} &
        \includegraphics[width=0.17\textwidth]{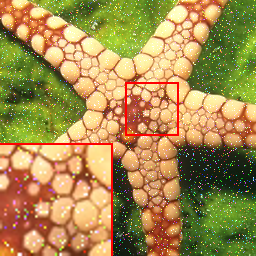} &
        \includegraphics[width=0.17\textwidth]{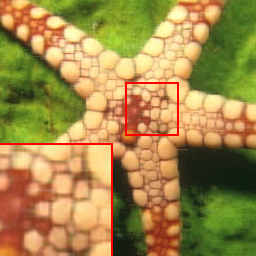} &
        \includegraphics[width=0.17\textwidth]{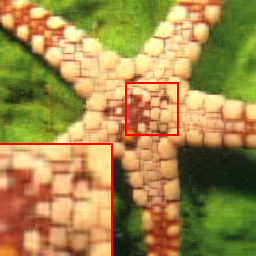} &
        \includegraphics[width=0.17\textwidth]{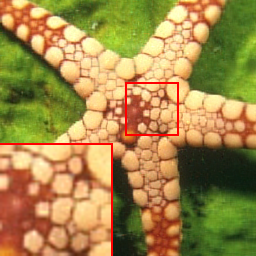} \\
        GT & NY & QNNM & QWNNM & QSNM$_{1/2}$ \\
        
        \includegraphics[width=0.17\textwidth]{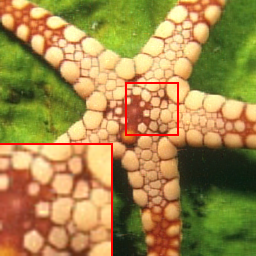} &
        \includegraphics[width=0.17\textwidth]{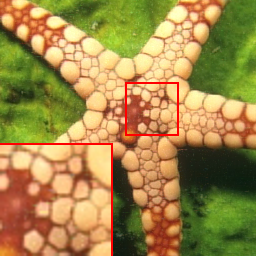} &
        \includegraphics[width=0.17\textwidth]{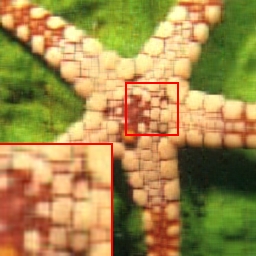} &
        \includegraphics[width=0.17\textwidth]{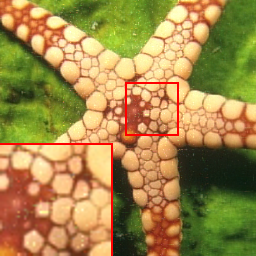} &
        \includegraphics[width=0.17\textwidth]{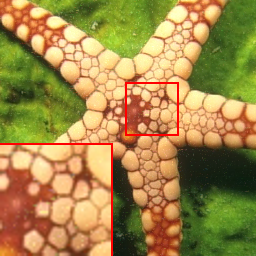} \\
         QSNM$_{2/3}$ & QSNM$_{0.975}$ & QWSNM & QNMF & QNOF 

	\end{tabular}
	}
	\caption{ Comparison of RPCA performance by random impulse noise removal from color images. The random rate is 10\%. }
	\label{rpca_10}
\end{figure}

\begin{figure}[!t]
	\centering
	\addtolength{\tabcolsep}{-5.5pt}
    \renewcommand\arraystretch{0.4}
	{\fontsize{8pt}{\baselineskip}\selectfont 
	\begin{tabular}{ccccc}
        \includegraphics[width=0.17\textwidth]{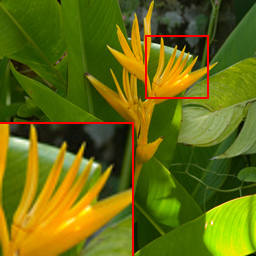} &
        \includegraphics[width=0.17\textwidth]{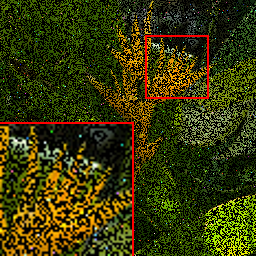} &
        \includegraphics[width=0.17\textwidth]{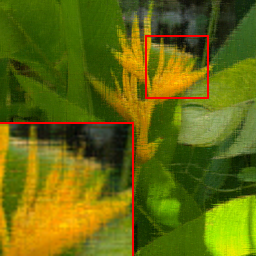} &
        \includegraphics[width=0.17\textwidth]{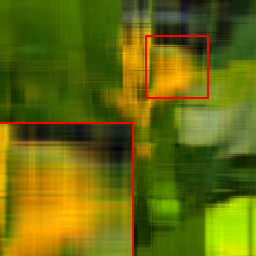} &
        \includegraphics[width=0.17\textwidth]{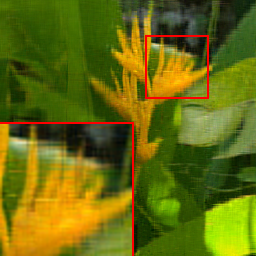} \\
        GT & NY & QNNM & QWNNM & QSNM$_{1/2}$ \\
        
        \includegraphics[width=0.17\textwidth]{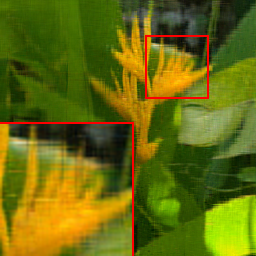} &
        \includegraphics[width=0.17\textwidth]{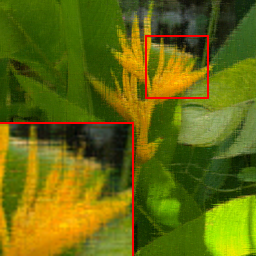} &
        \includegraphics[width=0.17\textwidth]{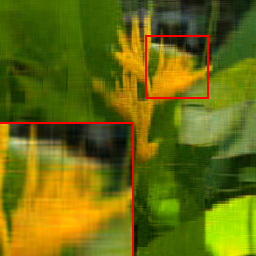} &
        \includegraphics[width=0.17\textwidth]{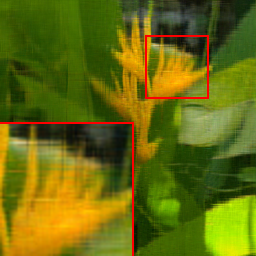} &
        \includegraphics[width=0.17\textwidth]{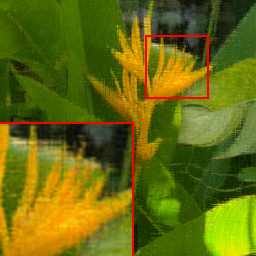} \\
         QSNM$_{2/3}$ & QSNM$_{0.975}$ & QWSNM & QNMF & QNOF 

	\end{tabular}
	}
	\caption{Comparison of RMC performance by color image inpainting combined with random impulse noise removal. Image missing rate of 50\% and noise random rate of 3\%.}
	\label{rmc_50_3}
\end{figure}

Figs. \ref{rpca_10} and \ref{rmc_50_3} show a visual comparison of different algorithms. These figures highlight that QNOF achieves superior reconstruction clarity and better detail preservation compared to other methods. In contrast, QWNNM and QWSNM exhibit cluttered reconstruction details, while QNNM suffers from noticeable over-smoothing. Specifically, in Fig. \ref{rpca_10}, QNMF leaves significant noise residuals, whereas QNOF produces fewer residuals. Overall, QNOF demonstrates state-of-the-art performance both numerically and visually.

To further assess QNOF's performance, Fig. \ref{error} presents the variation of PSNR, SSIM, and convergence error across iterations for different algorithms. The convergence error is computed as $\frac{|| \dot{\bm{X}}^{(k)} - \dot{\bm{X}}^{(k-1)}||_{F}}{|| \dot{\bm{X}}^{(k)} ||_{F}}$. The steady convergence of QNOF is evident in Figs. \ref{error}(a) and (b). The zoomed-in section of Fig. \ref{error}(c) shows that QNOF consistently achieves the lowest convergence error. Around the 40th iteration, QNOF meets the stopping criterion and halts further iterations, demonstrating faster convergence than other algorithms. Table \ref{tab:time} lists the runtime for RMC testing on a single $256 \times 256$ image, further confirming the results from Fig. \ref{error}(c).

\begin{figure}[t]
	\centering
	\addtolength{\tabcolsep}{-5.5pt}
    \renewcommand\arraystretch{0.5}
	{\fontsize{8pt}{\baselineskip}\selectfont 
	\begin{tabular}{ccc}
        \includegraphics[width=0.33\textwidth, trim={23 0 33 23}, clip]{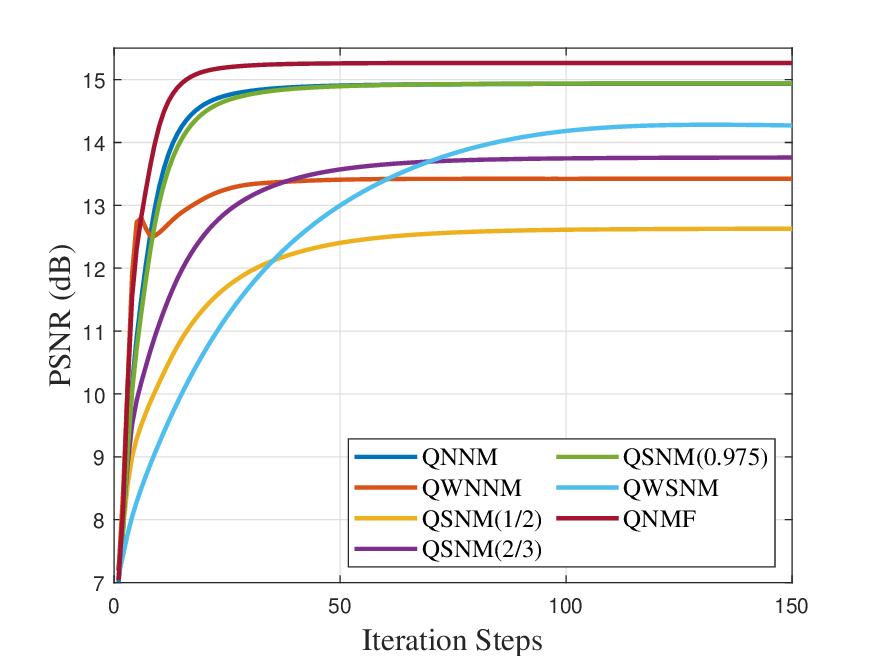} &
        \includegraphics[width=0.33\textwidth, trim={15 0 33 20}, clip]{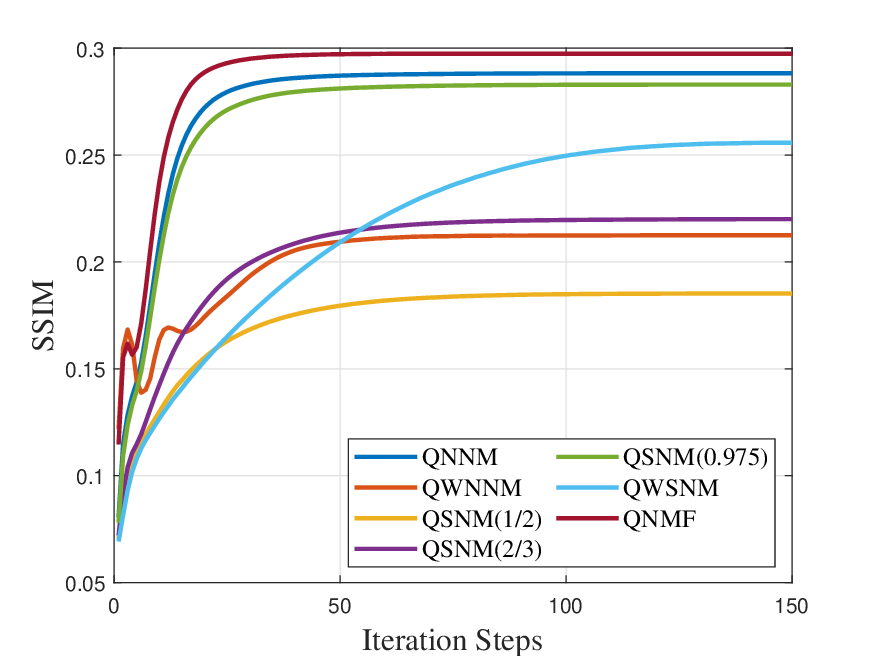} &
        \includegraphics[width=0.33\textwidth, trim={5 0 33 20}, clip]{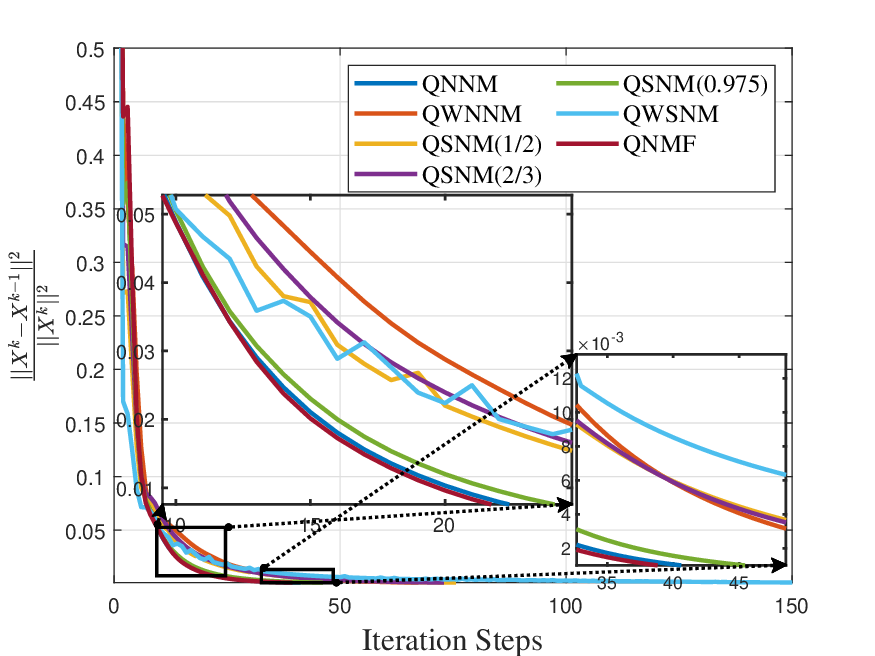} \\

        (a) PSNR & (b) SSIM & (c) $\frac{|| \dot{\bm{X}}^{(k)} - \dot{\bm{X}}^{(k-1)}||_{F}}{|| \dot{\bm{X}}^{(k)} ||_{F}}$ \\ 
	\end{tabular}
	}
	\caption{ PSNR, SSIM and error $\frac{|| \dot{\bm{X}}^{(k)} - \dot{\bm{X}}^{(k-1)}||_{F}}{|| \dot{\bm{X}}^{(k)} ||_{F}}$ change with iterations as an example of RMC with image missing rate 80\% and noise random rate 3\%. 
 }
	\label{error}
\end{figure}
\begin{table}[!t]
\footnotesize
\caption{State-of-the-art algorithms compare time(s) on RMC using $256\times256$ images.}
\label{tab:time}
\begin{center}
\resizebox{\textwidth}{!}{
  \begin{tabular}{ccccccccc} 
  \hline
    (Miss, Random)  & QNNM & QWNNM & QSNM$_{1/2}$ & QSNM$_{2/3}$ & QSNM$_{0.975}$ & QWSNM & QNMF & QNOF \\ \hline

     (50\%, 5\%)  & 10.44 & 9.56 & 19.41 & 18.57 & 13.08 & 19.28 & 6.90 & 6.20 \\ 
     (80\%, 3\%) & 10.35 & 9.62  & 18.65 & 18.03 & 14.28 & 18.20 & 12.74 & 4.96 \\ \hline

  \end{tabular}
}
\end{center}
\end{table}

\section{Conclusions}
\label{sec:conclusions}

In this paper, we introduce a novel nonconvex approximation to the rank of quaternion matrices, called the quaternion nuclear norm over Frobenius norm (QNOF). QNOF retains the scale invariance and parameterlessness of the rank function. Importantly, when applied to standard matrices, QNOF reduces to a special case where the quaternion's real part is preserved, and its three imaginary components are set to zero. By using QSVD, solving QNOF is simplified to the $L_1/L_2$ problem of determining the singular values of a quaternion matrix. We then present a solution to QNOF using the $L_1/L_2$ proximal operator.

We also extend QNOF to the robust matrix completion problem and propose a robust matrix completion model based on QNOF. This model can be efficiently solved using ADMM, with weak convergence guaranteed under mild conditions. Extensive experiments validate the effectiveness of our model compared to state-of-the-art approaches. The numerical results demonstrate that QNOF outperforms existing quaternion models in both accuracy and speed. Since SVD has $O(n^3)$ complexity, we plan to further investigate low-rank decomposition methods for quaternion matrices that avoid SVD, aiming to reduce the computational cost of the algorithm.

\appendix
\section{Proof of Theorem \ref{theorem04}} 
\label{appendix.a}

\begin{proof}[Proof]
First, we prove that the sequences $\{\dot{\bm{\eta}}^{(k+1)}\}$ and $\{\dot{\bm{\xi}}^{(k+1)}\}$ generated by Algorithm \ref{alg:rmc} have upper bounds. $\{\dot{\bm{\eta}}^{(k+1)}\}$ is related to QNOF, so
\begin{equation}
\begin{aligned}
||\dot{\bm{\eta}}^{(k+1)} ||_{F}^{2} &= || \dot{\bm{\eta}}^{(k)} + \beta_{1}^{(k)}(\dot{\bm{X}}^{(k+1)}-\dot{\bm{P}}^{(k+1)})  ||_{F}^{2} \\
&= {\beta_{1}^{(k)}}^{2} || \dot{\bm{X}}^{(k+1)} - ( \dot{\bm{P}}^{(k+1)} - {\beta_{1}^{(k)}}^{-1}\dot{\bm{\eta}}^{(k)})  ||_{F}^{2} \\
&= {\beta_{1}^{(k)}}^{2} ||\dot{\bm{U}}^{(k)}\Sigma_{\dot{\bm{X}}}^{(k)} \bm{V}^{(k)\top} -\dot{\bm{U}}^{(k)}\Sigma^{(k)} \bm{V}^{(k)\top} ||_{F}^{2} \\
&= {\beta_{1}^{(k)}}^{2} \sum^{n}_{i=1}(\sigma^{(k)}_{\dot{\bm{X}},i} - \sigma^{(k)}_{i})^2. 
\label{eta1}
\end{aligned}
\end{equation}
% 分两种情况
Next, two cases need to be discussed.
\begin{itemize}
\item (I) If $0< \dfrac{\beta_{1}^{(k)}}{\lambda} \leq \dfrac{1}{{\sigma^{(k)}_{1}}^{2}}$, by Theorem \ref{theorem001} and \ref{theorem0001}, it follows that $ \sigma^{(k)}_{\dot{\bm{X}},i}$ is 
$$ \sigma^{(k)}_{\dot{\bm{X}},i} = \left\{  
     \begin{array}{cl} \sigma^{(k)}_{i}, & i = 1;\\ 
     0, & i \neq 1. 
     \end{array}  
\right. $$
% Now, we can get $0< \dfrac{\beta_{1}^{(k)}}{\lambda} \leq \dfrac{1}{{\sigma^{(k)}_{1}}^{2}}$. 
Since the singular values are in descending order, we get
$$ \sum^{n}_{i=1}(\sigma^{(k)}_{\dot{\bm{X}},i} - \sigma^{(k)}_{i})^2 = \sum^{n}_{i=2} {\sigma^{(k)}_{i}}^2 \leq \sum^{n}_{i=2} {\sigma^{(k)}_{1}}^{2} = (n-1){\sigma^{(k)}_{1}}^{2}. $$
Because of $\frac{\beta_{1}^{(k)}}{\lambda} \leq \frac{1}{{\sigma^{(k)}_{1}}^{2}}$, we can obtain ${\sigma^{(k)}_{1}}^{2} \leq \frac{\lambda}{\beta_{1}^{(k)}}$. Thus, 
\begin{equation}
\begin{aligned}
||\dot{\bm{\eta}}^{(k+1)} ||_{F}^{2} \leq {\beta_{1}^{(k)}}^{2}(n-1){\sigma^{(k)}_{1}}^{2}.
\end{aligned}
\end{equation}
It can be further obtained that
\begin{equation}
\begin{aligned}
\dfrac{||\dot{\bm{\eta}}^{(k+1)} ||_{F}^{2}}{\theta^2} \leq {\beta_{1}^{(k)}}^{2}\dfrac{(n-1){\sigma^{(k)}_{1}}^{2}}{\theta^2} &\leq {\beta_{1}^{(k)}}^{2}\dfrac{(n-1){\sigma^{(k)}_{1}}^{4}}{\theta^4} \\
&= {\beta_{1}^{(k)}}^{2}\dfrac{(n-1){\lambda}^{4}}{\theta^{4}{\beta_{1}^{(k)}}^{2}} \\
&= \dfrac{(n-1){\lambda}^{4}}{\theta^{4}}, 
\end{aligned}
\end{equation}
where $\sigma_{1}^{(k)} > \theta$, $\theta>0$ is a constant, i.e. $\frac{\sigma_1}{\theta} > 1$. Hence, it follows that $\{\dot{\bm{\eta}}^{(k)}\}$ has an upper bound, i.e., 
\begin{equation}
\begin{aligned}
||\dot{\bm{\eta}}^{(k+1)} ||_{F}^{2} \leq \dfrac{(n-1){\lambda}^{4}}{\theta^{2}}.
\end{aligned}
\end{equation}

\item (II) If $\dfrac{\beta_{1}^{(k)}}{\lambda} > \dfrac{1}{{\sigma^{(k)}_{1}}^{2}}$, by Theorem \ref{theorem001} and \ref{theorem0001}, it follows that $ \sigma^{(k)}_{\dot{\bm{X}},i}$ is 
$$ \sigma^{(k)}_{\dot{\bm{X}},i} = \left\{  
     \begin{array}{cl} \dfrac{\beta_{1}^{(k)}/\lambda*\sigma^{(k)}_{i} - \dfrac{1}{r} }{ \beta_{1}^{(k)}/\lambda - \dfrac{a}{r^3} } , & 1 \leq i \leq t;\\ 
     0, & otherwise. 
     \end{array}  
\right. $$
From this, we can obtain
\begin{equation}
\begin{aligned}
&{\beta_{1}^{(k)}}^{2} \sum^{n}_{i=1}(\sigma^{(k)}_{\dot{\bm{X}},i} - \sigma^{(k)}_{i})^2 \\
&~= {\beta_{1}^{(k)}}^{2}(\sum^{t}_{i=1}( \frac{\beta_{1}^{(k)}/\lambda*\sigma^{(k)}_{i} - \frac{1}{r} }{ \beta_{1}^{(k)}/\lambda - \frac{a}{r^3} } - \sigma^{(k)}_{i})^2 + \sum^{n}_{i=t+1}{\sigma^{(k)}_{i}}^{2}) \\
% &~= {\beta_{1}^{(k)}}^{2}(\sum^{t}_{i=1}( \frac{\frac{a}{r^3}\sigma^{(k)}_{i} - \frac{1}{r} }{ \beta_{1}^{(k)}/\lambda - \frac{a}{r^3} })^2 + \sum^{n}_{i=t+1}{\sigma^{(k)}_{i}}^{2}) \\
&~= {\beta_{1}^{(k)}}^{2}(\sum^{t}_{i=1}( \frac{a\lambda\sigma^{(k)}_{i} - r^{2}\lambda }{ r^3\beta_{1}^{(k)}  -a\lambda })^2 + \sum^{n}_{i=t+1}{\sigma^{(k)}_{i}}^{2}) \\
&~\leq {\beta_{1}^{(k)}}^{2}(\sum^{t}_{i=1}( \frac{a\lambda\sigma^{(k)}_{1} - r^{2}\lambda }{ r^3\beta_{1}^{(k)}  -a\lambda })^2 + \sum^{n}_{i=t+1}{(\frac{\lambda}{r\beta_{1}^{(k)}}})^{2})  \\
&~=  {\beta_{1}^{(k)}}^{2} \left(t( \frac{a\lambda\sigma^{(k)}_{1} - r^{2}\lambda }{ r^3\beta_{1}^{(k)}  -a\lambda })^2\right) + (n-t)\frac{\lambda^2}{r^2} \\
% &~=  {\beta_{1}^{(k)}}^{2}{\color{red} \left(\frac{t}{{\beta_{1}^{(k)}}^{2}}( \frac{a\lambda\sigma^{(k)}_{1} - r^{2}\lambda }{ r^3  -\frac{a\lambda}{\beta_{1}^{(k)}} })^2\right)} + (n-t)\frac{\lambda^2}{r^2} \\
&~=  t(\frac{a\lambda\sigma^{(k)}_{1} - r^{2}\lambda }{ r^3  -\frac{a\lambda}{\beta_{1}^{(k)}} })^2 + (n-t)\frac{\lambda^2}{r^2} \\
\end{aligned}
\end{equation}
By $\sigma_1 \geq \sigma_i, i=1,\dots,t$ and \eqref{2.7}, the inequality holds, and it further follows that
\begin{equation}
\begin{aligned}
||\dot{\bm{\eta}}^{(k+1)} ||_{F}^{2} \leq  t(\frac{a\lambda\sigma^{(k)}_{1} - r^{2}\lambda }{ r^3  -\frac{a\lambda}{\beta_{1}^{(k)}} })^2 + (n-t)\frac{\lambda^2}{r^2}.
\end{aligned}
\label{eta_beta}
\end{equation}
Since $\{\dot{\bm{X}}^{(k)}\}$ is bounded, then we may as well set the upper bounds of its singular values $\Sigma_{\dot{\bm{X}}}$ of $0,1,2$-norms $t$, $a$, $r$ to be $M_t, M_a, M_r$. Here we omit the iterative superscript $(k)$ for $t$, $a$, $r$ for simplicity of representation. At the same time, $\dot{\bm{X}}$ does not degrade to $\bm{0}$, and the three small constants $\theta_t, \theta_a, \theta_r > 0$ are used to denote the lower bounds of $t$, $a$, $r$. According to the proof of \cite[Theorem 3.3]{tao2022minimization}, $\dfrac{\beta_{1}^{(k)}}{\lambda} > \dfrac{1}{{\sigma^{(k)}_{1}}^{2}}$ implies that $\dfrac{\beta_{1}^{(k)}}{\lambda} > \dfrac{a}{r^3}$, i.e. $\dfrac{\beta_{1}^{(k)}}{\lambda} > \dfrac{1}{{\sigma^{(k)}_{1}}^{2}} \geq \dfrac{a}{r^3}$. This gives the bound of $\sigma_{1}$ as 
\begin{equation}
\begin{aligned}
0< \theta < \sigma_{1}^{(k)} \leq \sqrt{\frac{r^3}{a}} < \sqrt{\frac{M_r^{3}}{\theta_a}}.
\end{aligned}
\end{equation}
From this, taking the limit of the inequality \eqref{eta_beta}, we obtain
\begin{equation}
\begin{aligned}
\lim_{k\rightarrow +\infty}||\dot{\bm{\eta}}^{(k+1)} ||_{F}^{2} &\leq  \lim_{k\rightarrow +\infty} t(\frac{a\lambda\sigma^{(k)}_{1} - r^{2}\lambda }{ r^3  -\frac{a\lambda}{\beta_{1}^{(k)}} })^2 + (n-t)\frac{\lambda^2}{r^2} \\
&= M_t(\frac{M_a\lambda\sqrt{\frac{M_r^{3}}{\theta_a}} - \theta_r^{2}\lambda }{ \theta_r^3 })^2 + (n-t)\frac{\lambda^2}{\theta_r^2}.
\end{aligned}
\end{equation}
\end{itemize}
In summary, $\{\dot{\bm{\eta}}^{(k+1)}\}$ is upper bounded.

Next, let's consider $\{\dot{\bm{\xi}}^{(k+1)}\}$. $\{\dot{\bm{\xi}}^{(k+1)}\}$ is related to $||\dot{\bm{Z}}||_{1}$, from which it can be concluded that
\begin{equation}
\begin{aligned}
||\dot{\bm{\xi}}^{(k+1)} ||_{F}^{2} &= || \dot{\bm{\xi}}^{(k)} + \beta_{2}^{(k)}(\dot{\bm{Z}}^{(k+1)}-\dot{\bm{Q}}^{(k+1)})  ||_{F}^{2} \\
&= {\beta_{2}^{(k)}}^{2} || \dot{\bm{Z}}^{(k+1)} - (\dot{\bm{Q}}^{(k+1)} - {\beta_{2}^{(k)}}^{-1}\dot{\bm{\xi}}^{(k)})  ||_{F}^{2} \\
&= {\beta_{2}^{(k)}}^{2} || \mathcal{S}_{\frac{\rho}{\beta_{2}^{(k)}}}(\dot{\bm{Q}}^{(k+1)} - {\beta_{2}^{(k)}}^{-1}\dot{\bm{\xi}}^{(k)}) - (\dot{\bm{Q}}^{(k+1)} - {\beta_{2}^{(k)}}^{-1}\dot{\bm{\xi}}^{(k)}) ||_{F}^{2} \\
&\leq {\beta_{2}^{(k)}}^{2} \sum_{i=1}^{n} (\rho/\beta_{2}^{(k)})^{2} \\
&= \rho^{2}n.
\label{xi}
\end{aligned}
\end{equation}
So, $\{\dot{\bm{\xi}}^{(k+1)}\}$ also has an upper bound.

% 接下来
Then, we prove that the sequence of Lagrange function $\{ \mathcal{L}(\dot{\bm{X}}^{(k)}, \dot{\bm{Z}}^{(k)}, \dot{\bm{P}}^{(k)}, \dot{\bm{Q}}^{(k)},$ $ \dot{\bm{\eta}}^{(k)},\dot{\bm{\xi}}^{(k)}, \beta_1^{(k)}, \beta_2^{(k)}) \}$ also has an upper bound. From the updated rules for $\dot{\bm{\eta}}^{(k)}$ and $\dot{\bm{\xi}}^{(k)}$, it follows that
\begin{equation}
\begin{aligned}
&\mathcal{L}(\dot{\bm{X}}^{(k+1)}, \dot{\bm{Z}}^{(k+1)}, \dot{\bm{P}}^{(k+1)}, \dot{\bm{Q}}^{(k+1)}, \dot{\bm{\eta}}^{(k+1)}, \dot{\bm{\xi}}^{(k+1)}, \beta_1^{(k+1)}, \beta_2^{(k+1)}) \\
&~=\mathcal{L}(\dot{\bm{X}}^{(k+1)}, \dot{\bm{Z}}^{(k+1)}, \dot{\bm{P}}^{(k+1)}, \dot{\bm{Q}}^{(k+1)}, \dot{\bm{\eta}}^{(k)}, \dot{\bm{\xi}}^{(k)}, \beta_1^{(k)}, \beta_2^{(k)}) \\  
&~~~+ \langle{\dot{\bm{\eta}}^{(k+1)} - \dot{\bm{\eta}}^{(k)}, \dot{\bm{X}}^{(k+1)}-\dot{\bm{P}}^{(k+1)} }\rangle  + \frac{\beta_1^{(k+1)}-\beta_1^{(k)}}{2}||\dot{\bm{X}}^{(k+1)}-\dot{\bm{P}}^{(k+1)} ||_{F}^{2} \\
&~~~+ \langle{\dot{\bm{\xi}}^{(k+1)} - \dot{\bm{\xi}}^{(k)}, \dot{\bm{Z}}^{(k+1)}-\dot{\bm{Q}}^{(k+1)} }\rangle  + \frac{\beta_2^{(k+1)}-\beta_2^{(k)}}{2}||\dot{\bm{Z}}^{(k+1)}-\dot{\bm{Q}}^{(k+1)} ||_{F}^{2} \\ % 第一个等式
&~=\mathcal{L}(\dot{\bm{X}}^{(k+1)}, \dot{\bm{Z}}^{(k+1)}, \dot{\bm{P}}^{(k+1)}, \dot{\bm{Q}}^{(k+1)}, \dot{\bm{\eta}}^{(k)}, \dot{\bm{\xi}}^{(k)}, \beta_1^{(k)}, \beta_2^{(k)}) \\   
&~~~+ \langle{\dot{\bm{\eta}}^{(k+1)} - \dot{\bm{\eta}}^{(k)}, \frac{\dot{\bm{\eta}}^{(k+1)} - \dot{\bm{\eta}}^{(k)}}{\beta_1^{(k)}} }\rangle  + \frac{\beta_1^{(k+1)}-\beta_1^{(k)}}{2}||\frac{\dot{\bm{\eta}}^{(k+1)} - \dot{\bm{\eta}}^{(k)}}{\beta_1^{(k)}} ||_{F}^{2} \\
&~~~+ \langle{\dot{\bm{\xi}}^{(k+1)} - \dot{\bm{\xi}}^{(k)}, \frac{\dot{\bm{\xi}}^{(k+1)} - \dot{\bm{\xi}}^{(k)}}{\beta_2^{(k)}} }\rangle  + \frac{\beta_2^{(k+1)}-\beta_2^{(k)}}{2}||\frac{\dot{\bm{\xi}}^{(k+1)} - \dot{\bm{\xi}}^{(k)}}{\beta_2^{(k)}} ||_{F}^{2} \\ % 第二个等式
&~=\mathcal{L}(\dot{\bm{X}}^{(k+1)}, \dot{\bm{Z}}^{(k+1)}, \dot{\bm{P}}^{(k+1)}, \dot{\bm{Q}}^{(k+1)}, \dot{\bm{\eta}}^{(k)}, \dot{\bm{\xi}}^{(k)}, \beta_1^{(k)}, \beta_2^{(k)}) \\   
&~~~+\frac{\beta_1^{(k+1)} +\beta_1^{(k)}}{2{\beta_1^{(k)}}^2}||\dot{\bm{\eta}}^{(k+1)} - \dot{\bm{\eta}}^{(k)} ||_{F}^{2} + \frac{\beta_2^{(k+1)} +\beta_2^{(k)}}{2{\beta_2^{(k)}}^2}||\dot{\bm{\xi}}^{(k+1)} - \dot{\bm{\xi}}^{(k)} ||_{F}^{2}. 
\label{L}
\end{aligned}
\end{equation}
Since $\{\dot{\bm{\eta}}^{(k)}\}$ and $\{\dot{\bm{\xi}}^{(k)}\}$ are bounded, $\{\dot{\bm{\eta}}^{(k+1)} - \dot{\bm{\eta}}^{(k)} \}$ and $\{\dot{\bm{\xi}}^{(k+1)} - \dot{\bm{\xi}}^{(k)} \}$ are also bounded. Suppose that the upper bound of the sequence $\{\dot{\bm{\eta}}^{(k+1)} - \dot{\bm{\eta}}^{(k)} \}$ is $M_1$, i.e., $ \forall k \geq 0, || \dot{\bm{\eta}}^{(k+1)} - \dot{\bm{\eta}}^{(k)} ||_F \leq M_1$. Similarly, suppose that the upper bound of the sequence $\{\dot{\bm{\xi}}^{(k+1)} - \dot{\bm{\xi}}^{(k)} \}$ is $M_2$, i.e., $ \forall k \geq 0, || \dot{\bm{\xi}}^{(k+1)} - \dot{\bm{\xi}}^{(k)} ||_F \leq M_2$. Meanwhile, the inequality $\mathcal{L}(\dot{\bm{X}}^{(k+1)}, \dot{\bm{Z}}^{(k+1)}, \dot{\bm{P}}^{(k+1)}, \dot{\bm{Q}}^{(k+1)}, \dot{\bm{\eta}}^{(k)}, \dot{\bm{\xi}}^{(k)}, \beta_1^{(k)}, \beta_2^{(k)}) \leq \mathcal{L}(\dot{\bm{X}}^{(k)}, \dot{\bm{Z}}^{(k)}, \dot{\bm{P}}^{(k)},$ $\dot{\bm{Q}}^{(k)}, \dot{\bm{\eta}}^{(k)}, \dot{\bm{\xi}}^{(k)}, \beta_1^{(k)}, \beta_2^{(k)})$ always holds because $\dot{\bm{X}}$, $\dot{\bm{Z}}$, $\dot{\bm{P}}$ and $\dot{\bm{Q}}$ are globally optimal solutions to the corresponding subproblems. Therefore, we have
\begin{equation}
\begin{aligned}
&\mathcal{L}(\dot{\bm{X}}^{(k+1)}, \dot{\bm{Z}}^{(k+1)}, \dot{\bm{P}}^{(k+1)}, \dot{\bm{Q}}^{(k+1)}, \dot{\bm{\eta}}^{(k+1)}, \dot{\bm{\xi}}^{(k+1)}, \beta_1^{(k+1)}, \beta_2^{(k+1)}) \\
&~\leq \mathcal{L}(\dot{\bm{X}}^{(k+1)}, \dot{\bm{Z}}^{(k+1)}, \dot{\bm{P}}^{(k+1)}, \dot{\bm{Q}}^{(k+1)}, \dot{\bm{\eta}}^{(k)}, \dot{\bm{\xi}}^{(k)}, \beta_1^{(k)}, \beta_2^{(k)}) + \frac{\beta_1^{(k+1)}+\beta_1^{(k)}}{2{\beta_1^{(k)}}^2}M_{1}^{2} + \frac{\beta_2^{(k+1)}+\beta_2^{(k)}}{2{\beta_2^{(k)}}^2}M_{2}^{2} \\ % 第一个不等式
\end{aligned}
% \label{L2}
\end{equation}
\begin{equation*}
\begin{aligned}
&~\leq \mathcal{L}(\dot{\bm{X}}^{(k)}, \dot{\bm{Z}}^{(k)}, \dot{\bm{P}}^{(k)}, \dot{\bm{Q}}^{(k)}, \dot{\bm{\eta}}^{(k)}, \dot{\bm{\xi}}^{(k)}, \beta_1^{(k)}, \beta_2^{(k)}) + \frac{\beta_1^{(k+1)}+\beta_1^{(k)}}{2{\beta_1^{(k)}}^2}M_{1}^{2} + \frac{\beta_2^{(k+1)}+\beta_2^{(k)}}{2{\beta_2^{(k)}}^2}M_{2}^{2} \\ % 第二个不等式
&~\leq \dots \\
&~\leq \mathcal{L}(\dot{\bm{X}}^{(1)}, \dot{\bm{Z}}^{(1)}, \dot{\bm{P}}^{(1)}, \dot{\bm{Q}}^{(1)}, \dot{\bm{\eta}}^{(0)}, \dot{\bm{\xi}}^{(0)}, \beta_1^{(0)}, \beta_2^{(0)}) + M_{1}^{2}\sum_{k=0}^{\infty}\frac{1+\mu}{2\beta_1^{(0)}\mu^{k}} + M_{2}^{2}\sum_{k=0}^{\infty}\frac{1+\mu}{2\beta_2^{(0)}\mu^{k}} \\
&~\leq \mathcal{L}(\dot{\bm{X}}^{(1)}, \dot{\bm{Z}}^{(1)}, \dot{\bm{P}}^{(1)}, \dot{\bm{Q}}^{(1)}, \dot{\bm{\eta}}^{(0)}, \dot{\bm{\xi}}^{(0)}, \beta_1^{(0)}, \beta_2^{(0)}) + \frac{M_{1}^{2}}{\beta_1^{(0)}}\sum_{k=0}^{\infty}\frac{1}{\mu^{k-1}} + \frac{M_{2}^{2}}{\beta_2^{(0)}}\sum_{k=0}^{\infty}\frac{1}{\mu^{k-1}}  < -\infty. \\
% &~
\end{aligned}
% \label{L2}
\end{equation*}
Hence, $\{ \mathcal{L}(\dot{\bm{X}}^{(k)}, \dot{\bm{Z}}^{(k)}, \dot{\bm{P}}^{(k)}, \dot{\bm{Q}}^{(k)}, \dot{\bm{\eta}}^{(k)},\dot{\bm{\xi}}^{(k)}, \beta_1^{(k)}, \beta_2^{(k)}) \}$ is  upper bounded.

Next, we consider the sequences $\{\dot{\bm{X}}^{(k)} \}$, $\{\dot{\bm{Z}}^{(k)} \}$, $\{\dot{\bm{P}}^{(k)} \}$ and $\{\dot{\bm{Q}}^{(k)} \}$, where $\{\dot{\bm{X}}^{(k)} \}$ is known to be bounded. From \eqref{E-rmc}, we get
\begin{equation}
\begin{aligned}
&\lambda(\frac{||\dot{\bm{X}}^{(k+1)}||_{*}}{||\dot{\bm{X}}^{(k+1)}||_{F}}) + \rho||\dot{\bm{Z}}^{(k+1)}||_{1} \\
&~=\mathcal{L}(\dot{\bm{X}}^{(k+1)}, \dot{\bm{Z}}^{(k+1)}, \dot{\bm{P}}^{(k+1)}, \dot{\bm{Q}}^{(k+1)}, \dot{\bm{\eta}}^{(k)}, \dot{\bm{\xi}}^{(k)}, \beta_1^{(k)}, \beta_2^{(k)}) - \langle{\dot{\bm{\eta}}^{(k)}, \dot{\bm{X}}^{(k+1)} - \dot{\bm{P}}^{(k+1)}}\rangle \\
&~~ - \frac{\beta_1^{(k)}}{2}|| \dot{\bm{X}}^{(k+1)} - \dot{\bm{P}}^{(k+1)} ||_{F}^{2} - \langle{\dot{\bm{\xi}}^{(k)}, \dot{\bm{Z}}^{(k+1)} - \dot{\bm{Q}}^{(k+1)}}\rangle  - \frac{\beta_2^{(k)}}{2}|| \dot{\bm{Z}}^{(k+1)} - \dot{\bm{Q}}^{(k+1)} ||_{F}^{2}  \\ % 第一个等式
&=\mathcal{L}(\dot{\bm{X}}^{(k+1)}, \dot{\bm{Z}}^{(k+1)}, \dot{\bm{P}}^{(k+1)}, \dot{\bm{Q}}^{(k+1)}, \dot{\bm{\eta}}^{(k)}, \dot{\bm{\xi}}^{(k)}, \beta_1^{(k)}, \beta_2^{(k)}) - \langle{\dot{\bm{\eta}}^{(k)}, \frac{\dot{\bm{\eta}}^{(k+1)} - \dot{\bm{\eta}}^{(k)}}{\beta_1^{(k)}} }\rangle  \\
&~~- \frac{\beta_1^{(k)}}{2}|| \frac{\dot{\bm{\eta}}^{(k+1)} - \dot{\bm{\eta}}^{(k)}}{\beta_1^{(k)}} ||_{F}^{2} - \langle{\dot{\bm{\xi}}^{(k)}, \frac{\dot{\bm{\xi}}^{(k+1)} - \dot{\bm{\xi}}^{(k)}}{\beta_2^{(k)}} }\rangle - \frac{\beta_2^{(k)}}{2}|| \frac{\dot{\bm{\xi}}^{(k+1)} - \dot{\bm{\xi}}^{(k)}}{\beta_2^{(k)}} ||_{F}^{2} \\ % 第二个等式
&= \mathcal{L}(\dot{\bm{X}}^{(k+1)}, \dot{\bm{Z}}^{(k+1)}, \dot{\bm{P}}^{(k+1)}, \dot{\bm{Q}}^{(k+1)}, \dot{\bm{\eta}}^{(k)}, \dot{\bm{\xi}}^{(k)}, \beta_1^{(k)}, \beta_2^{(k)}) + \frac{||\dot{\bm{\eta}}^{(k)}||_{F}^{2} - ||\dot{\bm{\eta}}^{(k+1)}||_{F}^{2} }{2\beta_1^{(k)}}\\
&~~+ \frac{||\dot{\bm{\xi}}^{(k)}||_{F}^{2} - ||\dot{\bm{\xi}}^{(k+1)}||_{F}^{2} }{2\beta_2^{(k)}}.
\label{XZ}
\end{aligned}
\end{equation}
From this, we obtain that $\{\dot{\bm{Z}}^{(k)} \}$ is upper bounded. Since
$$\dot{\bm{P}}^{(k+1)} = \dot{\bm{X}}^{(k+1)} - \dfrac{\dot{\bm{\eta}}^{(k+1)} - \dot{\bm{\eta}}^{(k)}}{\beta_{1}^{(k)}}, \quad \dot{\bm{Q}}^{(k+1)} = \dot{\bm{Z}}^{(k+1)} - \dfrac{ \dot{\bm{\xi}}^{(k+1)} - \dot{\bm{\xi}}^{(k)}}{ \beta_{2}^{(k)}},$$
the sequence $\{\dot{\bm{P}}^{(k)}\}$ and $\{\dot{\bm{Q}}^{(k)} \}$ are bounded. 
Thus, there exists at least one accumulation point for $\{ \dot{\bm{X}}^{(k)}, \dot{\bm{Z}}^{(k)},\dot{\bm{P}}^{(k)},\dot{\bm{Q}}^{(k)} \}$. Further, we can get
\begin{equation}
\begin{aligned}
\lim_{k\rightarrow +\infty} || \dot{\bm{X}}^{(k+1)}-\dot{\bm{P}}^{(k+1)} ||_{F} &= \lim_{k\rightarrow +\infty} \frac{1}{\beta_1^{(k)}}|| \dot{\bm{\eta}}^{(k+1)}-\dot{\bm{\eta}}^{(k)} ||_{F} = 0.
\end{aligned}
\end{equation}
Similarly,
\begin{equation}
\begin{aligned}
\lim_{k\rightarrow +\infty} || \dot{\bm{Z}}^{(k+1)}-\dot{\bm{Q}}^{(k+1)} ||_{F} =0.
\end{aligned}
\end{equation}
Thus, \eqref{converge35} and \eqref{converge36} have to be proved.

Finally, we demonstrate that the change in the sequences $\{\dot{\bm{X}}^{(k)}\}$, $\{\dot{\bm{Z}}^{(k)}\}$, $\{\dot{\bm{P}}^{(k)}\}$ and $\{\dot{\bm{Q}}^{(k)}\}$ converges to 0 with iteration. For $\dot{\bm{Z}}$ subproblem, by $\dot{\bm{Z}}^{(k)} =\dot{\bm{Q}}^{(k)} + \frac{\dot{\bm{\xi}}^{(k)} -\dot{\bm{\xi}}^{(k-1)}}{\beta_2^{(k-1)}}$, we get
\begin{equation}
\begin{aligned}
&\lim_{k\rightarrow +\infty} ||  \dot{\bm{Z}}^{(k+1)} - \dot{\bm{Z}}^{(k)} ||_{F} \\
&~= \lim_{k\rightarrow +\infty} || \mathcal{S}_{\frac{\beta_2^{(k)}}{\rho}}(\dot{\bm{Q}}^{(k)}  - \frac{\dot{\bm{\xi}}^{(k)}}{\beta_2^{(k)}})  - (\dot{\bm{Q}}^{(k)} + \frac{\dot{\bm{\xi}}^{(k)} -\dot{\bm{\xi}}^{(k-1)}}{\beta_2^{(k-1)}})||_{F} \\
&~= \lim_{k\rightarrow +\infty} || \mathcal{S}_{\frac{\beta_2^{(k)}}{\rho}}(\dot{\bm{Q}}^{(k)}  - \frac{\dot{\bm{\xi}}^{(k)}}{\beta_2^{(k)}})  - (\dot{\bm{Q}}^{(k)}  - \frac{\dot{\bm{\xi}}^{(k)}}{\beta_2^{(k)}})  + (\frac{\dot{\bm{\xi}}^{(k-1)}}{\beta_2^{(k-1)}} - \frac{\dot{\bm{\xi}}^{(k)}}{\beta_2^{(k-1)}} - \frac{\dot{\bm{\xi}}^{(k)}}{\beta_2^{(k)}})||_{F} \\
&~\leq \lim_{k\rightarrow +\infty} || \mathcal{S}_{\frac{\beta_2^{(k)}}{\rho}}(\dot{\bm{Q}}^{(k)}  - \frac{\dot{\bm{\xi}}^{(k)}}{\beta_2^{(k)}})  - (\dot{\bm{Q}}^{(k)}  - \frac{\dot{\bm{\xi}}^{(k)}}{\beta_2^{(k)}}) ||_{F}  + ||\frac{\dot{\bm{\xi}}^{(k-1)}}{\beta_2^{(k-1)}} - \frac{\dot{\bm{\xi}}^{(k)}}{\beta_2^{(k-1)}} - \frac{\dot{\bm{\xi}}^{(k)}}{\beta_2^{(k)}}||_{F} \\
&~\leq \lim_{k\rightarrow +\infty} \frac{\rho \sqrt{mn}}{\beta_2^{(k)}} + ||\frac{\mu\dot{\bm{\xi}}^{(k-1)} - (\mu+1) \dot{\bm{\xi}}^{(k)}}{\beta_2^{(k)}}||_{F} \\
&~=0,
\end{aligned}
\end{equation}
where $\mathcal{S}_{\frac{\beta_2^{(k)}}{\rho}}$ is the soft-threshold operation with parameter $\frac{\beta_2^{(k)}}{\rho}$, $m$ and $n$ are the size of $\dot{\bm{Y}}$. Thus, \eqref{converge32} is proved.  Then, we discuss the sequence $\{\dot{\bm{Q}}^{(k)}\}$.
\begin{equation}
\begin{aligned}
&\lim_{k\rightarrow +\infty} ||  \dot{\bm{Q}}^{(k+1)} - \dot{\bm{Q}}^{(k)} ||_{F} \\
&~= \lim_{k\rightarrow +\infty} || \dot{\bm{Z}}^{(k+1)}  - \frac{\dot{\bm{\xi}}^{(k+1)} -\dot{\bm{\xi}}^{(k)}}{\beta_2^{(k)}} - \dot{\bm{Q}}^{(k)}||_{F} \\
&~= \lim_{k\rightarrow +\infty} || \dot{\bm{Z}}^{(k+1)} - \dot{\bm{Z}}^{(k)} + \dot{\bm{Z}}^{(k)} - \dot{\bm{Q}}^{(k)}  - \frac{\dot{\bm{\xi}}^{(k+1)} -\dot{\bm{\xi}}^{(k)}}{\beta_2^{(k)}} ||_{F} \\
&~\leq \lim_{k\rightarrow +\infty} || \dot{\bm{Z}}^{(k+1)} - \dot{\bm{Z}}^{(k)} ||_{F}  + || \dot{\bm{Z}}^{(k)} - \dot{\bm{Q}}^{(k)} ||_{F} + || \frac{\dot{\bm{\xi}}^{(k)} -\dot{\bm{\xi}}^{(k+1)}}{\beta_2^{(k)}} ||_{F} \\
&~=0.
\end{aligned}
\end{equation}
\eqref{converge34} is proved.

Sequences $\{\dot{\bm{P}}^{(k)}\}$ and $\{\dot{\bm{X}}^{(k)}\}$ are discussed in the same way. For $\dot{p}_{i,j} \in \Omega$ satisfying the constraint $\mathcal{P}_{\Omega}(\dot{\bm{P}} + \dot{\bm{Q}}) = \mathcal{P}_{\Omega}(\dot{\bm{Y}})$, it is evident that 
$$
\lim_{k\rightarrow +\infty} ||  \mathcal{P}_{\Omega}(\dot{\bm{P}}^{(k+1)}) - \mathcal{P}_{\Omega}(\dot{\bm{P}}^{(k)}) ||_{F} = \lim_{k\rightarrow +\infty} ||  \mathcal{P}_{\Omega}(\dot{\bm{Y}} - \dot{\bm{Q}}^{(k+1)}) - \mathcal{P}_{\Omega}(\dot{\bm{Y}} - \dot{\bm{Q}}^{(k)} ) ||_{F} = 0
$$
holds.  Let's discuss the situation of $\dot{p}_{i,j} \in \bar{\Omega}$.
\begin{equation}
\begin{aligned}
&\lim_{k\rightarrow +\infty} ||  \mathcal{P}_{\bar{\Omega}}(\dot{\bm{P}}^{(k+1)}) - \mathcal{P}_{\bar{\Omega}}(\dot{\bm{P}}^{(k)}) ||_{F}\\
&~= \lim_{k\rightarrow +\infty} || \mathcal{P}_{\bar{\Omega}}(\dot{\bm{X}}^{(k)}) + \frac{\mathcal{P}_{\bar{\Omega}}(\dot{\bm{\eta}}^{(k)})}{\beta_1^{(k)}} -  \mathcal{P}_{\bar{\Omega}}(\dot{\bm{X}}^{(k)}) + \frac{ \mathcal{P}_{\bar{\Omega}}(\dot{\bm{\eta}}^{(k)} - \dot{\bm{\eta}}^{(k-1)}) }{\beta_1^{(k-1)}} ||_{F} \\
&~\leq \lim_{k\rightarrow +\infty} || \frac{\mathcal{P}_{\bar{\Omega}}(\dot{\bm{\eta}}^{(k)})}{\beta_1^{(k)}}  ||_{F} + || \frac{ \mathcal{P}_{\bar{\Omega}}(\dot{\bm{\eta}}^{(k)} - \dot{\bm{\eta}}^{(k-1)}) }{\beta_1^{(k-1)}}  ||_{F} \\
&~=  0.
\end{aligned}
\end{equation}
In summary, $\lim_{k\rightarrow +\infty} ||  \dot{\bm{P}}^{(k+1)} - \dot{\bm{P}}^{(k)} ||_{F} =0$. For the sequence $\{\dot{\bm{X}}^{(k)}\}$,
\begin{equation}
\begin{aligned}
&\lim_{k\rightarrow +\infty} ||  \dot{\bm{X}}^{(k+1)} - \dot{\bm{X}}^{(k)} ||_{F} \\
&~= \lim_{k\rightarrow +\infty} || \dot{\bm{P}}^{(k+1)}  + \frac{\dot{\bm{\eta}}^{(k+1)} -\dot{\bm{\eta}}^{(k)}}{\beta_1^{(k)}} - \dot{\bm{X}}^{(k)}||_{F} \\
&~= \lim_{k\rightarrow +\infty} || \dot{\bm{P}}^{(k+1)} - \dot{\bm{P}}^{(k)}  + \frac{\dot{\bm{\eta}}^{(k+1)}  -\dot{\bm{\eta}}^{(k)}}{\beta_1^{(k)}} + \dot{\bm{P}}^{(k)} - \dot{\bm{X}}^{(k)}||_{F} \\
&~\leq \lim_{k\rightarrow +\infty} || \dot{\bm{P}}^{(k+1)} - \dot{\bm{P}}^{(k)} ||_{F}  + || \dot{\bm{P}}^{(k)} - \dot{\bm{X}}^{(k)} ||_{F} + || \frac{\dot{\bm{\eta}}^{(k+1)} -\dot{\bm{\eta}}^{(k)}}{\beta_1^{(k)}} ||_{F} =0, 
\end{aligned}
\end{equation}
\eqref{converge31} and \eqref{converge33} are proved.

\end{proof}

% \section*{Acknowledgments}
% We would like to acknowledge the assistance of volunteers in putting
% together this example manuscript and supplement.

\bibliographystyle{siamplain}
\bibliography{references}
\end{document}